\documentclass[lettersize,journal]{IEEEtran}

\ifCLASSOPTIONcompsoc
  \usepackage[nocompress]{cite}
\else
  \usepackage{cite}
\fi

\usepackage{amsthm}
\usepackage{amsmath}
\usepackage{amssymb}
\usepackage{color}

\newtheorem{lemma}{Lemma}
\newtheorem{theorem}{Theorem}

\newtheorem{definition}{Definition}

\newtheorem{remark}{Remark}

\DeclareMathOperator*{\argmin}{argmin}
\DeclareMathOperator*{\argmax}{argmax}

\usepackage{graphicx}
\usepackage{subfigure}
\usepackage{paralist}
\usepackage{multirow}

\usepackage{algorithm}
\usepackage{algorithmic}

\usepackage{bm}
\usepackage{booktabs}
\usepackage{verbatim}
\usepackage{array}
\usepackage{diagbox}
\usepackage{balance}
\usepackage{makecell}
\usepackage{ragged2e}
\usepackage{enumitem}
\DeclareMathAlphabet\mathbfcal{OMS}{cmsy}{b}{n}
\setlist[itemize]{leftmargin=*,topsep=0em}
\setlist[enumerate]{itemsep=0em, partopsep=0em}
\usepackage{hyperref}

\begin{document}

\title{Provable Unrestricted Adversarial Training without Compromise with Generalizability}

\author{Lilin~Zhang,~
        Ning~Yang,~
        Yanchao~Sun,~
        Philip S.~Yu,~\IEEEmembership{Fellow,~IEEE}
\IEEEcompsocitemizethanks{
\IEEEcompsocthanksitem Lilin Zhang is with the School
of Computer Science, Sichuan University, China. \protect 
 E-mail: zhanglilin@stu.scu.edu.cn
 
\IEEEcompsocthanksitem Ning Yang is the corresponding author and with the School
of Computer Science, Sichuan University, China. \protect 
 E-mail: yangning@scu.edu.cn
 
 \IEEEcompsocthanksitem Yanchao Sun is with the Department of Computer Science, University of Maryland, College Park, USA.  \protect 
 E-mail: ycsun1531@gmail.com

\IEEEcompsocthanksitem Philip S. Yu is with the Department of Computer Science, University of Illinois at Chicago, USA. \protect 
E-mail: psyu@uic.edu
}
}

\markboth{Journal of \LaTeX\ Class Files,~Vol.~14, No.~8, August~2015}%
{Shell \MakeLowercase{\textit{et al.}}: Bare Demo of IEEEtran.cls for Computer Society Journals}
\IEEEpubid{0000--0000/00\$00.00~\copyright~2021 IEEE}

\maketitle

\begin{abstract}

Adversarial training (AT) is widely considered as the most promising strategy to defend against adversarial attacks and has drawn increasing interest from researchers. However, the existing AT methods still suffer from two challenges. First, they are unable to handle unrestricted adversarial examples (UAEs), which are built from scratch, as opposed to restricted adversarial examples (RAEs), which are created by adding perturbations bound by an $l_p$ norm to observed examples. Second, the existing AT methods often achieve adversarial robustness at the expense of standard generalizability (i.e., the accuracy on natural examples) because they make a tradeoff between them. To overcome these challenges, we propose a unique viewpoint that understands UAEs as imperceptibly perturbed unobserved examples. Also, we find that the tradeoff results from the separation of the distributions of adversarial examples and natural examples. Based on these ideas, we propose a novel AT approach called Provable Unrestricted Adversarial Training (PUAT), which can provide a target classifier with comprehensive adversarial robustness against both UAE and RAE, and simultaneously improve its standard generalizability. Particularly, PUAT utilizes partially labeled data to achieve effective UAE generation by accurately capturing the natural data distribution through a novel augmented triple-GAN. At the same time, PUAT extends the traditional AT by introducing the supervised loss of the target classifier into the adversarial loss and achieves the alignment between the UAE distribution, the natural data distribution, and the distribution learned by the classifier, with the collaboration of the augmented triple-GAN. Finally, the solid theoretical analysis and extensive experiments conducted on widely-used benchmarks demonstrate the superiority of PUAT.

\end{abstract}


\section{Introduction}
\label{sec:introduction}
\IEEEPARstart{D}{eep} neural network (DNN) has achieved ground-breaking success in a number of artificial intelligence application areas, including computer vision, object identification, and natural language processing, among others. Despite the remarkable success of DNN, seminal researches reveal its vulnerability to \textbf{adversarial examples} (AEs), which are inputs with non-random perturbations unnoticeable to humans but intentionally meant to cause victim models to deliver false outputs \cite{biggio2013evasion,szegedy2013intriguing,ilyas2019adversarial,miller2020adversarial}. The harms brought on by AEs have prompted ongoing efforts to increase DNN's \textbf{adversarial robustness} (i.e., the accuracy on AEs), among which \textbf{adversarial training} (AT) has been largely acknowledged as the most promising defensive strategy \cite{bai2021recent,akhtar2021advances,zhao2022adversarial}.

The fundamental idea of AT was first proposed by Szegedy \textit{et al}. \cite{szegedy2013intriguing}, which adds AEs to the training set and uses a min-max game to create a robust classifier that can withstand worst-case attacks. In particular, AT alternately maximizes an adversarial loss incurred by AEs that are purposefully crafted using adversarial attacking methods such as FGSM \cite{goodfellow2014explaining} and PGD \cite{madry2017towards}, and minimizes it by adjusting the target classifier's parameters via a regular supervised training on the augmented training set. Numerous works have suggested strengthening the basic AT in order to increase adversarial robustness \cite{wang2019bilateral,wang2019improving,cheng2020cat,shafahi2020universal,lee2017generative,stutz2019disentangling}. Although the existing AT approaches have advanced significantly, we argue that the problem of adversarial robustness is still far from being well solved partly due to the following two challenges.
\IEEEpubidadjcol



\begin{figure}[!t]
\centering
		\subfigure[Adversarial robustness against RAE]{\includegraphics[scale=0.30]{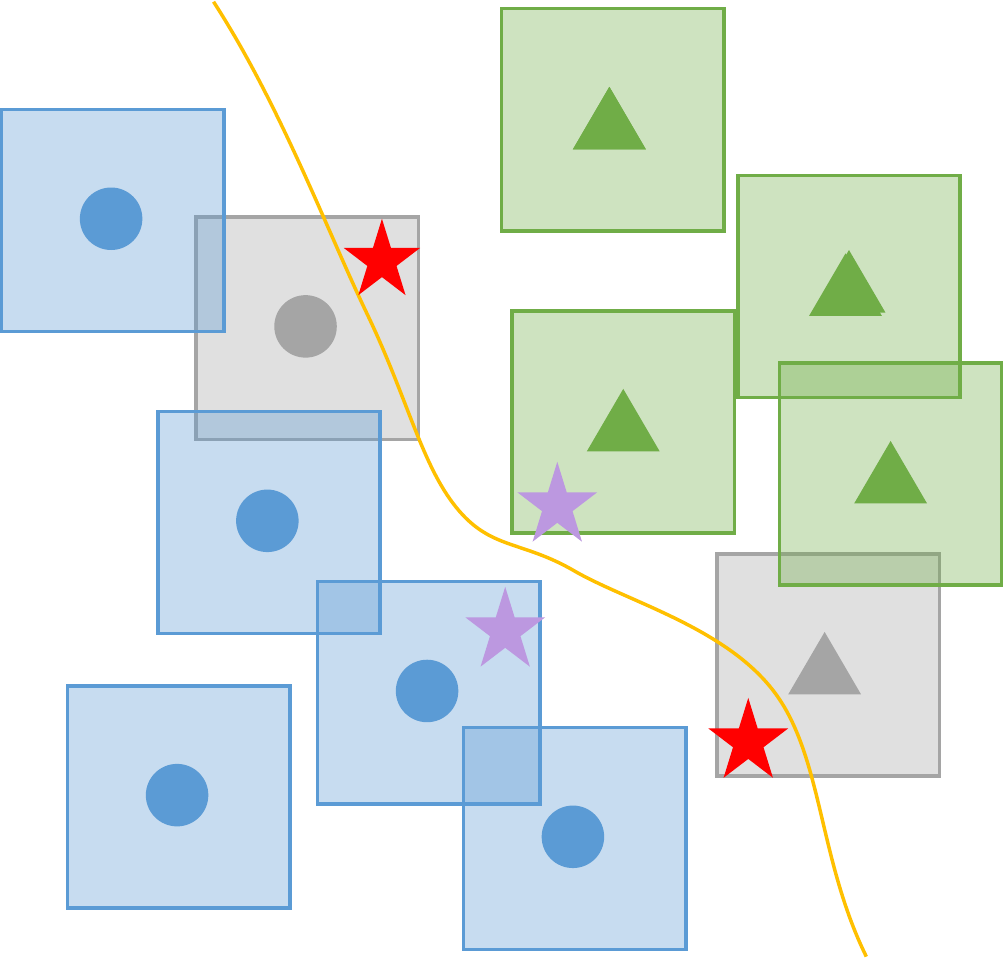} \label{fig:example_ae}} \quad\quad
		\subfigure[Adversarial robustness against UAE]{\includegraphics[scale=0.30]{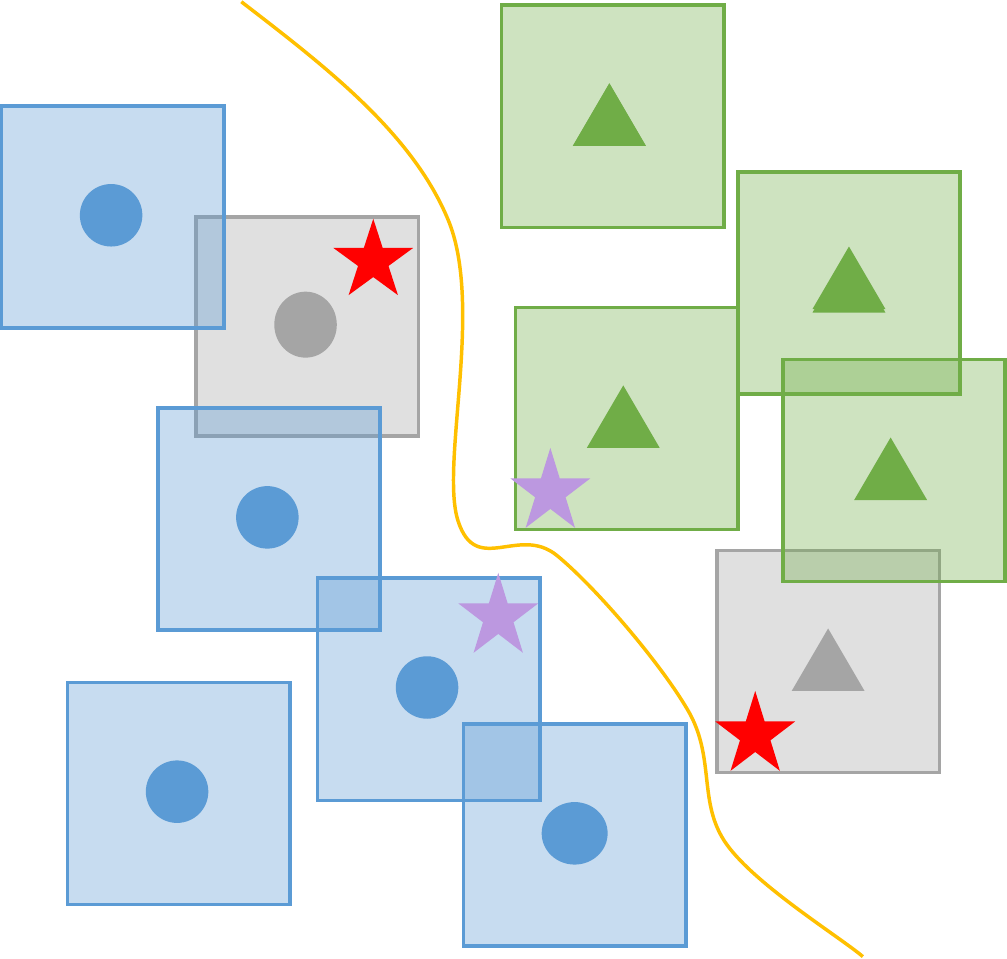} \label{fig:example_uae}}
\caption{Illustration of adversarial robustness against RAE and UAE. The blue dots and green triangles represent two classes of observed examples, respectively, and the gray ones represent the unobserved examples of the two classes. Each example is surrounded by a box representing its neighborhood with small distance. The purple stars represent RAEs while the red stars UAEs. The yellow curves are decision boundaries obtained by adversarial training.}
\label{fig:illustration}
\end{figure}

The issue of \textit{Unrestricted Adversarial Example} (UAE) is the first obstacle. As previously mentioned, AT is conducted on a training set that includes AEs. However, the conventional AT approaches are frequently limited to perturbation-based AEs, which are created by adding adversarial perturbations with magnitudes restricted by an $l_p$ norm to observed natural examples. We refer to such kind of AEs as restricted AE (RAE). In fact, RAEs only account for a small portion of possible AEs in real world. Recent works reveal the existence of UAEs, which are dissimilar to any observable example since they are built from scratch rather than by perturbing existing examples \cite{song2018constructing,dunn2019adaptive,bhattad2020unrestricted,naderi2022generating,Tao2022EGM}. Figure \ref{fig:illustration} provides an illustration showing how UAEs (represented by the red stars) are not the neighbors of any observed example (blue circle or green triangle), in contrast to RAEs (shown by the purple stars). Despite the adversarial robustness against RAE provided by conventional AT approaches, as seen in Figure \ref{fig:example_ae}, the classifier (with the yellow curve as the decision boundary) can still be deceived by UAEs because they also appear natural \cite{song2018constructing}. UAEs are more numerous, more covert, and more dangerous than RAEs since 
UAEs are generated by perturbing the samples from the learned natural sample distribution, which may or may not be observable. \label{intro-revision}
However, the existing works on UAE mainly focus on the generation of UAEs, which suffer from two defects. First, no explanation for why UAEs exist and why they are imperceptible to humans is provided in the existing works. Second, adversarial robustness against UAE still remains largely unexplored. Therefore, we need a new AT method that is able to offer the robustness against both UAE and RAE, which we refer to as \textbf{comprehensive adversarial robustness}. As shown in Figure \ref{fig:example_uae}, the classifier with comprehensive adversarial robustness is able to correctly predict the labels of both the UAEs and the RAEs.

The second challenge faced by the current AT approaches is the problem of the \textit{deterioration of standard generalizability} (i.e., the accuracy on natural/natural testing examples). Early researches show that the standard generalizability of a classifier is sacrificed in order for conventional AT approaches to obtain the adversarial robustness for the classifier \cite{madry2017towards}. Zhang \textit{et al}. \cite{tsipras2018robustness} further demonstrate the existence of an upper-bound of the standard generalizability given the adversarial robustness, and Attias \textit{et al}. \cite{attias2022improved} suggest lowering the upper-bound by increasing sample complexity. However, recent studies argue that the tradeoff may not be inevitable, and adversarial robustness and standard generalizability can coexist without negatively affecting each other \cite{stutz2019disentangling,yang2020closer,Song2020Robust,xing2021generalization}. For example, Stutz \textit{et al}. \cite{stutz2019disentangling} demonstrate that the adversarial robustness against the AEs existing on the manifold of natural examples will equivalently lead to better generalization. Song \textit{et al}. \cite{Song2020Robust} and Xing \textit{et al}. \cite{xing2021generalization} propose to improve the standard generalizability by introducing robust local features and $l_1$ penalty to AT, respectively. 
At the same time, some other works \cite{zhang2019defense,carmon2019unlabeled,alayrac2019labels} show that the tradeoff bound can be improved by utilizing unlabeled data in AT.
The existing efforts, however, only concentrate on RAEs. It is yet unknown that if and how the comprehensive adversarial robustness and the standard generalizability can both be improved simultaneously.

To address the first challenge, in this paper, we propose \textit{a unique viewpoint that understands UAEs as imperceptibly perturbed unobserved examples} (gray dot or triangle in Figure \ref{fig:illustration}), which logically explains why UAEs are dissimilar to observed examples and why UAEs can fool a classifier without confusing humans. Essentially, RAE can be thought of as a special UAE resulting from perturbing observed examples. If we generalize the concept of UAE to imperceptibly perturbed natural examples, regardless of whether they are observable, then \textit{UAE can be regarded as a general form of AE}, which makes it feasible to provide comprehensive adversarial robustness. 

For the second challenge, we find that the tradeoff exists because the target classifier can not generalize over two separated distributions, the AE distribution and the natural example distribution, and the distribution learned by the classifier is a mix of them. 
Therefore, if we can leverage partial labeled data to align the distributions of UAEs and natural examples with the distribution learned by the target classifier, then the classifier can generalize over UAEs and natural examples without conflict, and we can expect to achieve a better the tradeoff between comprehensive adversarial robustness and standard generalization. 

Based on these ideas, we propose a novel AT method called Provable Unrestricted Adversarial Training (PUAT), which \textit{integrates the generation of UAEs and an extended AT on UAEs with a GAN-based framework, so that the distributions of UAEs and natural examples can be aligned with the distribution learned by the target classifier for simultaneously improving its comprehensive adversarial robustness and standard generalizability}. At first, we design a UAE generation module consisting of an attacker $A$ and a conditional generator $G$, treating a UAE as an adversarially perturbed natural example. Different from the existing generative models for AEs, which often result in suboptimal AEs due to the gradient conflicts during the single optimization for both imperceptibility and harmfulness of AEs, our UAE generation module decouples the optimizations of imperceptibility and harmfulness of UAEs. Particularly, $A$ is in charge of producing an adversarial perturbation that can fool the classifier, based on which $G$ will generate a UAE for a certain class, which is equivalent to perturbing a natural example. PUAT ensures the imperceptibility of UAEs by aligning the UAE distribution learned by $G$ with the natural example distribution via a G-C-D GAN, in contrast to conventional AT approaches that accomplish imperceptibility by an explicit norm constraint on the perturbations. The G-C-D GAN is a novel augmented triple-GAN that combines two conditional GANs: the G-D GAN, which consists of $G$ and a discriminator $D$, and the C-D GAN, which consists of the target classifier $C$ and $D$. Different from the original triple-GAN proposed by Li \textit{et al}. \cite{li2017triple}, in G-C-D GAN the generator $G$ is shared with the UAE generation module, which augments the input of $G$ with the output of the attacker $A$ and consequently allows $G$ to draw UAEs from the UAE distribution. To facilitate the distributional alignment, we also make use of unlabeled data to create pseudo-labeled data for the training of G-C-D GAN since labeled data are usually insufficient for capturing the real data distribution. By leveraging partially labeled data, G-D GAN and C-D GAN can cooperate to help $G$ capture the natural data distribution and the UAE generation module understand how to build a legitimate AE for a certain class.

In order to simultaneously improve the target classifier's comprehensive adversarial robustness and standard generalizability, PUAT conducts an extended AT between the target classifier $C$ and the attacker $A$, which incorporates a supervised loss of the target classifier with the adversarial loss. 
The extended AT together with the G-C-D GAN leads to the alignment between the UAE distribution, the natural example distribution, and the distribution learned by $C$, which enables $C$ to generalize without conflict on UAEs and natural examples. Finally, we theoretically and empirically demonstrate that PUAT realizes a consistent optimization for both the comprehensive adversarial robustness and the standard generalizability, eliminating the tradeoff between them and resulting in their mutual advantage. Our contributions can be summarized as follows:
\begin {itemize}

\item We propose a unique viewpoint that understands UAEs as imperceptibly perturbed unobserved examples, which offers a logical and demonstrable explanation for UAEs' discrepancies from observed examples, imperceptibility to humans, and the feasibility of comprehensive adversarial robustness against both RAE and UAE.



\item We propose a novel AT approach called Provable Unrestricted Adversarial Training (PUAT), which, with theoretical guarantee, can simultaneously increase a classifier's comprehensive adversarial robustness and standard generalizability through a distributional alignment.


\item We provide a solid theoretical analysis to demonstrate that PUAT can eliminate the tradeoff between the comprehensive adversarial robustness and the standard generalizability through the distributional alignment.

\item The extensive experiments conducted on widely adopted benchmarks verify the superiority of PUAT.

\end {itemize}

\begin{figure*}[!t]
\centering
\includegraphics[width=5.0in]{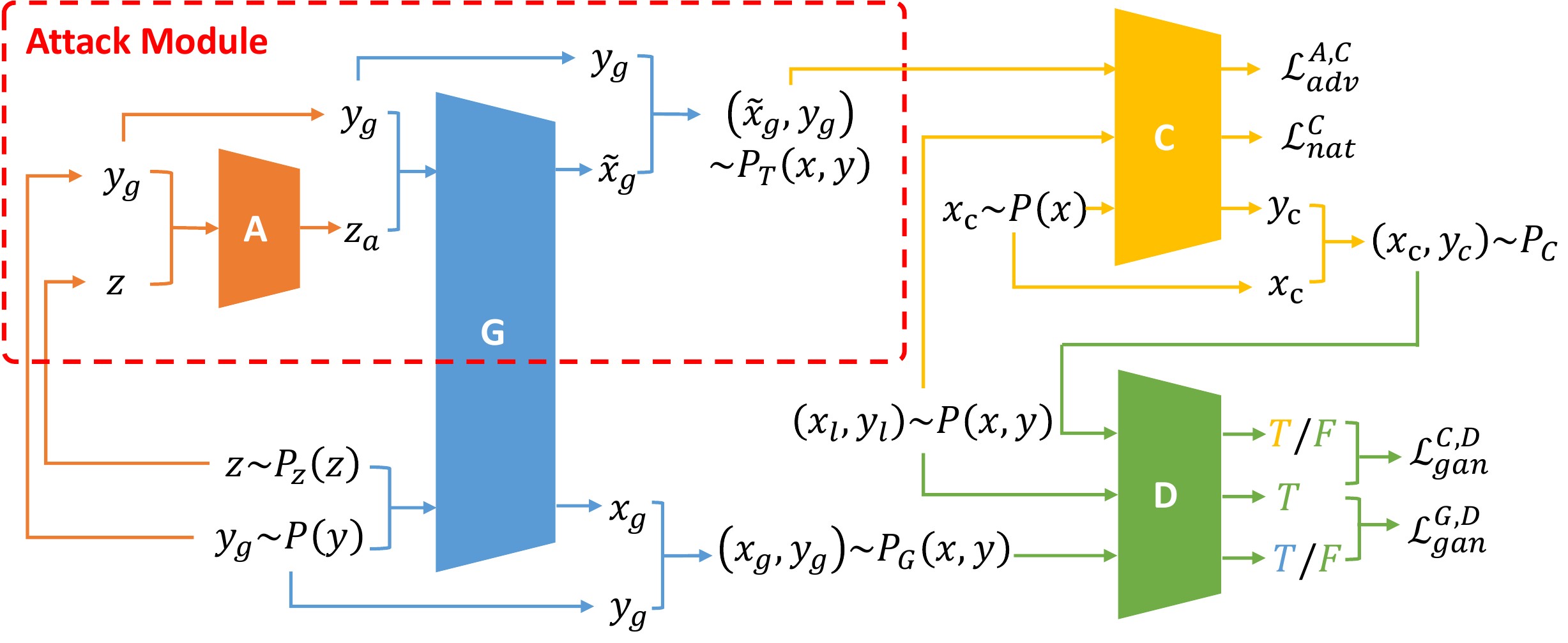}
\caption{Architecture of PUAT. The attacker $A$ seeks the perturbation $z_a$, and the generator $G$ synthesizes the UAEs $\{ (\tilde{x}_g, y) \}$ based on the perturbation and the specified class $y$. In G-D GAN, the discriminator $D$ aims to adversarially distinguish the unobserved natural examples $\{ (x_g, y) \}$ generated by $G$ from the true natural examples $\{ (x, y)\}$ with loss $\mathcal{L}^{G,D}_{gan}$, while in C-D GAN, $D$ aims to adversarially distinguish the pseudo labeled examples $\{ (x, y_c)\}$ generated by the target classifier $C$ from $\{ (x, y)\}$ with loss $\mathcal{L}^{C,D}_{gan}$. The extended AT is conducted between $A$ and $C$ with the adversarial loss $\mathcal{L}^{C,A}_{adv}$ plus the supervised loss $\mathcal{L}^{C}_{nat}$.}
\label{Fig:model}
\end{figure*}

\section{Preliminaries and Problem Formulation}
In this section, we first give the formal definition of UAE, then present our understanding of UAEs, and finally, formulate the AT on UAEs.
\subsection{Definition of UAE}

Let $x \in \mathcal{X}$ and $y \in \mathcal{Y}$ be an example and a label, respectively, where $ \mathcal{X}$ represents the data space with distribution $P(x)$ while $\mathcal{Y}$ the label space with distribution $P(y)$. Let $o(\cdot) : \mathcal{O} \subseteq \mathcal{X} \rightarrow \mathcal{Y}$ be an oracle which is defined on its domain $\mathcal{O}$ and tells the true label of an example $x \in \mathcal{O}$. Usually, the oracle $o$ corresponds to humans and $\mathcal{O}$ consists of all the samples that look natural to and can not confuse humans. Let $C(\cdot) : \mathcal{X} \rightarrow \mathcal{Y}$ be the target classifier. Following the idea of \cite{song2018constructing}, we define UAE as follow:
\begin{definition}[UAE]
For a well-trained classifier $C$, an example $\tilde{x}$ with label $y = o(\tilde{x})$ is an unrestricted adversarial example if $\tilde{x} \sim P(x|y) $ and $C(\tilde{x}) \neq y$, where $P(x|y)$ is the class likelihood.
\label{Def:UAE}
\end{definition}

\subsection{
Understanding of UAE
}
\label{Sec:understanding}
From Definition \ref{Def:UAE} we can see that similar to RAE, the harmfulness of a UAE is still defined by its ability to fool the classifier, i.e., $C(\tilde{x}) \neq y$. But in contrast to RAE, the imperceptibility of a UAE is due to its capability of being drawn from the realistic distribution ($\tilde{x} \sim  P(x|y) $) rather than a norm-constraint. However, it is challenging how to enforce both imperceptibility and harmfulness for UAE. For this purpose, \textit{we understand UAE from the viewpoint of imperceptibly perturbed observed and unobserved examples}. Inspired by \cite{staib2017distributionally}, for a natural sample $x \sim P(x|y)$, there exists a mapping $T:\mathcal{X}\times\mathcal{Y}\rightarrow\mathcal{X}$ to generate a UAE $\tilde{x} = T(x, y)$ satisfying the imperceptibility $\tilde{x}  \in supp(P(x|y))$, and the harmfulness $C(\tilde{x}) \neq y$, where $supp()$ represents the support set. Following this idea, we can generate a UAE through two steps: drawing natural example and generating appropriate perturbation for it. It is noteworthy that such understanding brings the advantage that RAE can be viewed as a special UAE and UAEs are a superset of RAEs, thus enabling the adversarial robustness against UAEs to cover RAEs.

\subsection{
Adversarial Training on UAEs
}
\label{Sec:ATonUAEs}
To offer the comprehensive adversarial robustness to the target classifier, we want to conduct adversarial training on UAEs, which is defined as the following min-max game:
\begin{equation}
\begin{aligned}
&\min_{C}  \max_{P_T} \mathbb{E}_{(\tilde{x},y)\sim {P}_T(x,y)} l\big(\tilde{x}, y;C \big),
\end{aligned}
\label{Eq:Ladv}
\end{equation}
where $l(x, y; C)$ is a loss function of the target classifier $C$, e.g., cross-entropy, $\tilde{x} = T(x, y)$ is a UAE with label $y$, and $P_T(x, y)$ is the labeled UAE distribution. In Equation (\ref{Eq:Ladv}), the inner maximization seeks the most harmful UAE set resulting in the poorest performance of the target classifier, while the outer minimization aims at adjusting the target classifier to reduce its sensitivity to UAEs. It is noteworthy that unlike traditional AT, where AEs are searched by explicitly perturbed the observed natural samples with respect a norm-constraint, the AT defined by Equation (\ref{Eq:Ladv}) is conducted on UAEs which are perturbed versions of not only observed natural samples but also unobserved ones, leading to superior comprehensive adversarial robustness to traditional AT methods as shown by later theoretical analysis and experimental verification. To fulfill the AT on UAEs, the key challenge is how to learn the adversarial distribution $P_T(x, y)$ that can reconcile the adversarial robustness and standard generalizability of the target classifier. To overcome this challenge, our general idea is to fulfill the learning of distribution ${P}_T(x, y)$ with two steps: (1) approximating the distribution $P(x, y)$ of the natural samples, and (2) learning the mapping $T$, which converts the inner maximization of Equation (\ref{Eq:Ladv}) to 
\begin{equation}
\begin{aligned}
\max_{T} \mathbb{E}_{(x,y)\sim P(x,y)} l\bigg(\big(\tilde{x}=T(x,y), y \big) ;C \bigg). \\
\end{aligned}
\label{Eq:Ladv_}
\end{equation}
When $T$ achieves the optimal solution of Equation (\ref{Eq:Ladv_}), ${P}_T(x, y)$ converges to the AE distribution $P(x, y|C(x) \ne y)$.

\section{Proposed Method}
In this section, we first give an overview of our Provable Unrestricted Adversarial Training (PUAT) method, and then describe its components and learning procedure in detail.

\subsection{Overview of PUAT}

%
%
%

As shown in Figure \ref{Fig:model}, PUAT consists of four components, an attacker $A$, a conditional generator $G$, a discriminator $D$ and the target classifier $C$. According to our idea that a UAE can be regarded as the perturbed version of an observed or unobserved natural sample, we first use $G$ to produce a plausible natural sample $(x_g, y_g) \sim P_G(x, y)$, where $P_G(x, y)$ is the distribution of the data-label pairs generated by $G$ and the randomness stems from a white noise $z$. Then we invoke $A$ and $G$ to generate its perturbed version $(\tilde{x}_g, y_g)$ as a UAE, where the perturbation $z_a$ is produced by $A$ based on the same $z$. Note that conceptually, $A$ and $G$ together play the role of afore-mentioned mapping $T$ to generate a UAE from the distribution $P_T(x,y)$, which we term as \textbf{attack module}, and in other words, $P_T(x,y)$ is exactly the distribution captured by $A$ together with $G$.


To ensure the \textbf{harmfulness} of the UAEs, PUAT will try to adjust $A$ via the AT between $A$ and $C$ so that the generated UAEs can fool the target classifier $C$, i.e. $C(\tilde{x}_g) \ne y_g$. At the same time, to ensure the \textbf{imperceptibility} of the generated UAEs, PUAT will align the distribution $P_G(x, y)$ with the true data distribution $P(x, y)$. For this purpose, we introduce a discriminator $D$ to distinguish the examples $\{ (x_g, y_g) \}$, which includes the UAEs $\{ (\tilde{x}_g, y_g) \}$, from the true labeled examples $\{ (x_l, y_l) \sim P(x, y) \}$, which together with $G$ constitutes a class-conditional GAN, called G-D GAN. The optimization of the G-D GAN will adjust $G$ so that the distribution $P_{G}(x, y)$ converges to the natural sample distribution $P(x, y)$. The alignment of $P_{G}(x, y)$ and $P(x, y)$ will make a UAE drawn from $P_G(x, y)$ look like a real natural sample, which results in the imperceptibility of UAEs. Finally, it is noteworthy that the harmfulness and imperceptibility offered by PUAT cause the UAEs essentially follow the conditional distribution $P(x, y|C(x) \ne y)$.

However, in traditional conditional GAN, the discriminator $D$ also plays the role to predict the class label of $x_g$, for which the optimization objective conflicts with that for distinguishing $\{ (x_g, y_g) \} $ from $\{ (x_l, y_l) \}$. Such conflict will lead to a suboptimal $D$ \cite{li2017triple,Li2022}, and consequently weaken the ability of the G-D GAN to accurately align $P_{G}(x, y)$ with $P(x,y)$. To alleviate this issue, we introduce another conditional GAN, called C-D GAN, which together with the G-D GAN constitutes the G-C-D GAN. The C-D GAN consists of the target classifier $C$ and the discriminator $D$, where $C$ aims at predicting the label $y_c$ for unlabeled examples $\{ x_c \sim P(x)\}$ to build pseudo-labeled examples $\{ (x_c, y_c)\}$ that can fool $D$, while $D$ tries to distinguish $\{ (x_c, y_c) \} $ from true examples $\{ (x_l, y_l) \}$. The adversarial game between $C$ and $D$ improves the ability of $D$ to identify whether a label matches an example $x$. As $D$'s role in C-D GAN is consistent with its role in G-D GAN, $D$ strengthened by C-D GAN will ultimately help $G$ capture the true data distribution $P(x, y)$ more accurately.

At last, to realize the \textbf{comprehensive adversarial robustness} of the target classifier $C$, we want to align the $P_G(x, y)$ with the distribution $P_{C}(x, y)$ learned by $C$, which equivalently improves the robust generalizability of $C$ on UAEs. For this purpose, we conduct AT between the attacker $A$ and $C$, with an adversarial loss $\mathcal{L}^{A,C}_{adv}$ over the generated UAEs. At the same time, to simultaneously improve the \textbf{standard generalizability} of $C$ on natural examples, we further enforce the distribution $P_{C}(x, y)$ to be aligned with the true data distribution $P(x, y)$, by extending the AT with minimization of the supervised loss $\mathcal{L}^{C}_{sup}$ of $C$ over the labeled data $\{ (x_l, y_l) \}$. The overall logic of PUAT can be summarized as follows:
\begin{enumerate}[label =(\arabic*), topsep=0pt, partopsep=0pt]
\setlength{\parskip}{0pt}
\item For the imperceptibility of UAEs, $P_{G}(x, y)$ is aligned with the true data distribution $P(x, y)$, denoted by $P_{G}(x, y) = P(x, y)$, via the G-C-D GAN.
\item The harmfulness of UAEs is offered by the maximization of $\mathcal{L}^{A,C}_{adv}$ during the AT.
\item The adversarial robustness is offered by $P_{C}(x, y) = P_{T}(x, y)$, via the inner minimization of $\mathcal{L}^{A,C}_{adv}$ of the AT between $A$ and $C$.
\item The standard generalizability is offered by $P_{C}(x, y) = P(x, y)$, via the minimization of $\mathcal{L}^{C}_{sup}$.
\end{enumerate} 

It is noteworthy that traditional AT methods only realize (3) and (4), which causes $P_{C}(x, y)$ to be a mix of $P_T(x, y)$ and $P(x, y)$ and consequently incurs an inferior tradeoff between adversarial robustness and standard generalizability. In sharp contrast with the traditional methods, PUAT improves the tradeoff because (1) together with (3) and (4) results in the consistency between adversarial robustness and standard generalizability at $P_{C}(x, y) = P_{G}(x, y) = P(x, y)$, which will be proved later.

\subsection{UAE Generation}
\label{sec:UAE}
Our idea to generate a UAE $(\tilde{x}_g, y_g)$ is to perturb a generated natural example. Let $P(y)$ be the label distribution and $P_{z}(z)$ be a standard normal distribution. At first, we sample a label $y_g \sim P(y)$, and a seed noise $z \sim P_{z}(z)$. Then we invoke $G$ to generate an example $(x_g, y_g) \sim P_G(x, y)$, where
\begin{equation}
x_g = G(z, y_g).
\label{Eq:x_g}
\end{equation}
Note that since the label distribution $P(y)$ is unknown, we use the labels appearing in the labeled dataset $\mathcal{D}_l = \{ (x_l, y_l) \sim P(x, y)\}$ to approximate it.   In particular, the label $y_l$ of each example in $\mathcal{D}_l$ will serve as a $y_g$ once.

To generate the UAE corresponding to $(x_g, y_g)$, we further feed $(z, y_g)$ into the attacker $A$ to produce the perturbation, 
\begin{equation}
z_a = A(z, y_g).
\label{Eq:z_a}
\end{equation}
Next, $z_a$ together with $y_g$ is fed into the generator $G$ to produce perturbed data
\begin{equation}
\tilde{x}_g = G(z_a, y_g).
\label{Eq:tilde}
\end{equation}
Finally, $\tilde{x}_g$ and $y_g$ together form a labeled UAE $(\tilde{x}_g, y_g)$. 

It is noteworthy that that $(x_g, y_g)$ can be regarded as an unobserved natural example if $P_G(x, y)$ is aligned with the true data distribution $P(x, y)$. For each $(x_g, y_g)$, we will invoke the attack module ($A$ and $G$) generate a corresponding UAE $(\tilde{x}_g, y_g)$, which is actually the perturbed $(x_g, y_g)$ as they are generated based on the same $(z, y_g)$ with the only difference that whether $z$ or its perturbed version $z_a = A(z, y_g)$ is fed into $G$. At the same time, as we will see later, during the AT between $A$ and $C$, $A$ will be adjusted to ensure $(\tilde{x}_g, y_g)$ be an adversarial example that is able to attack the classifier $C$, i.e., $C(\tilde{x}_g) \neq y_g$. Therefore, the UAEs $\{ (\tilde{x}_g, y_g) \}$ are a subset of $\{ (x_g, y_g) \}$ such that $(\tilde{x}_g, y_g) \sim P(x, y|C(x) \ne y)$, and essentially, $A$ plays the role of finding out $\{ (\tilde{x}_g, y_g) \}$ from $\{ (x_g, y_g) \}$.

\subsection{Distribution Alignment for Imperceptibility}

As mentioned before, the imperceptibility of UAEs results from the alignment of the distribution $P_G(x, y)$ with the true data distribution $P(x, y)$, which is realized through the cooperation of G-D GAN and C-D GAN. Let $\mathcal{D}_l = \{ (x_l, y_l) \sim P(x,y)\}$ be a labeled natural dataset, where $y_l = o(x_l)$, and $\mathcal{Z}$ be a set of sampled noises $z\sim P_z(z)$. Then the optimization objective of G-D GAN is defined as 
\begin{equation}
\min_{G}\max_{D} \mathcal{L}^{G,D}_{gan} = \mathcal{L}_{D} + \mathcal{L}_{G},
\label{Eq:L_GD}
\end{equation}
where 
\begin{equation}
\begin{aligned}
\mathcal{L}_{D} &= \mathbb{E}_{(x_l, y_l) \sim P(x, y)} D(x_l, y_l) \\ 
&\approx \hat{\mathcal{L}}_{D} = \frac{1}{|\mathcal{D}_l |}\sum_{(x_l, y_l) \in \mathcal{D}_l} D(x_l, y_l), 
\end{aligned}
\label{Eq:LD}
\end{equation}

\begin{equation}
\begin{aligned}
 \mathcal{L}_{G} &= \mathbb{E}_{(x_g, y_g) \sim P_{G}(x, y)} \left[- D(x_g, y_g)\right]\\
&\approx  \hat{\mathcal{L}}_{G} = \frac{1}{\vert \mathcal{D}_l \vert\vert \mathcal{Z} \vert} \sum_{ y_l \in \mathcal{D}_l} \sum_{z \in \mathcal{Z}} \left[- D  \Big(G \big(z, y_l \big), y_l \Big) \right].
 \end{aligned}
 \label{Eq:LG}
\end{equation}
Note that in Equations (\ref{Eq:LD}) and (\ref{Eq:LG}), the first line is the population loss defined on the overall distribution, while the second line is the empirical loss defined on training set which approximates the population loss. In Equation (\ref{Eq:LG}), the label $y_l$ of each example $(x_l, y_l) \in \mathcal{D}_l$ serves as $y_g \sim P(y)$. For each $y_l$, we first sample a noise $z \sim P_{z}(z)$ and then invoke $G \big(z, y_l \big)$ to generate $x_g$, which together with $y_l$ forms a data-label pair $(x_g, y_g)$.

In G-D GAN, the generation of a data-label pair $(x_g, y_g)$ can be understood as the procedure of first drawing $y_g \sim P(y)$ and then drawing $x_g|y_g \sim P_G(x|y)$, where $P_G(x|y)$ is the conditional distribution characterized by $G$. Therefore, $(x_g, y_g)$ $ \sim$ $ P_G(x, y)$ $ = P(y)P_G(x|y)$. As $y_g$ is drawn from the true label distribution $P(y)$, to adversarially distinguish $(x_g, y_g)$ from true examples, what $D$ should do is to identify whether $x_g|y_g$ is true. There is one of two reasons causing a false $x_g|y_g$. One is $x_g$ is false, which asks $D$ to identify whether $x_g \sim P(x)$. The other is although $x_g$ is true, it does not belong to class $y_g$, which asks $D$ to identify whether $y_g|x_g \sim P(y|x)$. However, as we have mentioned before, these two tasks conflict with each other for $D$, which leads to a suboptimal $D$ and consequently weakens $G$. To alleviate this issue, we introduce C-D GAN and  can use an unlabeled natural dataset $\mathcal{D}_{c} = \{x_c \sim P(x) \}$ to train it with the following optimization objective:
\begin{equation}
\min_{C}\max_{D} \mathcal{L}^{C,D}_{gan} = \mathcal{L}_{D} + \mathcal{L}_{C},
\label{Eq:L_CD}
\end{equation}
where 
\begin{equation}
\begin{aligned}
 \mathcal{L}_{C} &= \mathbb{E}_{(x_c, y_c) \sim P_{C}(x, y)} \left[- D(x_c, y_c)\right] \\
&\approx \hat{\mathcal{L}}_{C} = \frac{1}{\vert \mathcal{D}_c \vert}\sum_{x_c \in \mathcal{D}_c} \left[- D\big(x_c, C(x_c)\big)\right].
 \end{aligned}
 \label{Eq:LC}
\end{equation}
To calculate $\mathcal{L}_{C}$, for each $x_c \in \mathcal{D}_c$, we invoke the target classifier $C$ to label it with the softmax vector $y_c = C(x_c) \in \mathbb{R}^{|\mathcal{Y}|}$. This is equivalent to the procedure of first drawing $x_c \sim P(x)$ and then drawing $y_c|x_c \sim P_{C}(y|x)$. Therefore $(x_c, y_c) \sim P_{C}(x, y) = P(x)P_{C}(y|x) $. C-D GAN offers two benefits to $D$. The first is in C-D GAN, $C$ provides the adversarial signal $y_c|x_c$ to help $D$ learn $P(y|x)$, which ultimately promotes $D$'s performance in G-D GAN for a better $G$. Thus in both C-D GAN and G-D GAN, $D$ plays the same role of distinguishing the generated examples ($\{ (x_g, y_g) \}$ or $\{ (x_c, y_c) \}$) from true ones ($\{ (x_l, y_l) \}$), which eliminates the conflict optimization of $D$ in traditional conditional GAN. The second is by C-D GAN, unlabeled data can be leveraged to accurately approximate $P(x,y)$.

We refer to the combination of G-D GAN and C-D GAN as G-C-D GAN. By combining Equations (\ref{Eq:L_GD}) and (\ref{Eq:L_CD}), we obtain the following optimization objective for G-C-D GAN:
\begin{equation}
\min_{C, G}\max_{D} \mathcal{L}^{G,C,D}_{gan} = \mathcal{L}_{D} + \frac{1}{2}\mathcal{L}_{G} + \frac{1}{2}\mathcal{L}_{C}.
\label{Eq:L_gan}
\end{equation}

\subsection{Extended Adversarial Training}
Since the harmfulness of the generated UAEs is caused by the attacker $A$, we conduct the AT between $A$ and the target classifier $C$. Meanwhile, to improve both the adversarial robustness and standard generalizability, we extend the AT defined in Equation (\ref{Eq:Ladv}) by incorporating the loss of $C$ on natural data, which leads to the following min-max game:
\begin{equation}
\min_{C} \big[\mathcal{L}^C_{nat} + \lambda \max_{A}\mathcal{L}^{C,A}_{adv} \big],
\label{Eq:AT}
\end{equation}
where $\lambda$ is a non-negative constant controlling the weight of adversarial robustness and the standard generalizability. $\mathcal{L}^C_{nat} $ is the loss of target classifier $C$ over natural samples. To benefit from unlabeled natural samples, here we introduce a consistent loss term offered by a weight-averaged classifier, following the idea of the semi-supervised algorithm Mean Teacher \cite{tarvainen2017mean}, which leads to the following definition:
\begin{equation}
\begin{aligned}
\mathcal{L}^{C}_{nat} =& \mathbb{E}_{(x, y) \sim P(x, y) } \big[-\log P_C(y|x) \big] \\
&+ \alpha \mathbb{E}_{x \sim P(x) } \big[\Vert P_C(y|x) - P_{C^\prime}(y|x) \Vert^2 \big]\\
\approx &  \hat{\mathcal{L}}^{C}_{nat} =  \frac{1}{\vert \mathcal{D}_l \vert}\sum_{(x_l, y_l) \in \mathcal{D}_l} \big[-\log P_C(y_l|x_l) \big]\\
&+ \alpha \frac{1}{\vert \mathcal{D}_c \vert} \big[\Vert P_C(y|x_c) - P_{C^\prime}(y|x_c) \Vert^2 \big],
\end{aligned}
\label{Eq:sup_loss}
\end{equation}
where the conditional probability $P_C(y|x)$ is the output of $C$, $C^\prime$ is the weight-averaged classifier, and $\alpha$ is the weight of consistency cost.

The adversarial loss is defined as
\begin{equation}
\begin{aligned}
\mathcal{L}^{C,A}_{adv} & = \mathbb{E}_{(\tilde{x},y)\sim P_T(x,y)}[-\log P_C(y|\tilde{x})]\\
& \approx \hat{\mathcal{L}}^{C,A}_{adv} = \frac{1}{\vert \mathcal{D}_l \vert\vert \mathcal{Z} \vert} \sum_{ y_l \in \mathcal{D}_l} \sum_{z \in \mathcal{Z}} \big[-\log P_C(y_l \vert \tilde{x}) \big],
\end{aligned}
\label{Eq:adver_loss}
\end{equation}
where $P_T(x,y)$ is the UAE distribution captured by the attack module, and $\tilde{x}$ is the UAE generated by $\tilde{x}=G(A(z,y_l), y_l)$. Note that for each $(x, y) \sim P_{G}(x,y)$, we generate its corresponding UAE $(\tilde{x}, y) $ by invoking Equations (\ref{Eq:z_a}) and (\ref{Eq:tilde}) (see Section \ref{sec:UAE}), i.e., $\tilde{x}=G(A(z,y),y)$. 

In Equation (\ref{Eq:AT}), the inner maximization of $\mathcal{L}^{C,A}_{adv}$ enforces $A$ to be adjusted so that a generated UAE $(\tilde{x}, y)$ is harmful enough, i.e., $C(\tilde{x}) \neq y$, which causes the attack module (consisting of $A$ and $G$) to converge to the underlying UAE distribution $P(x, y|C(x) \neq y)$ which maximizes the adversarial loss. The outer minimization realizes the alignment $P_{C}(x, y) = P_T (x, y)$ by adjusting $C$, which results in $C$'s robust generalizability over UAEs (i.e. the adversarial robustness). As proved later, the KL-divergence of $P_{C}(x, y)$ and $P_T (x, y)$ is the upper bound of the KL-divergence of $P_{C}(x, y)$ and $P_{G}(x, y)$. Therefore, the minimization of $\mathcal{L}^{C,A}_{adv}$ will approximately lead to $P_{C}(x, y) = P_{G}(x, y)$. At the same time, one can note that the minimization of $\mathcal{L}^{C}_{nat}$ will be achieved when $P_{C}(x, y) = P(x, y)$, which offers $C$ the standard generalizability on natural examples. Later we will also prove that with the help of the G-C-D GAN (Equation(\ref{Eq:L_gan})), the extended AT defined by Equation(\ref{Eq:AT}) will achieve the optimal solution at $P_C(x, y) = P_G(x, y) = P(x, y)$, which means the adversarial robustness and the standard generalizability are satisfied simultaneously.

\subsection{Joint Training}
By combining Equations (\ref{Eq:L_gan}) and (\ref{Eq:AT}), we can obtain the following overall optimization objective for joint learning of $A$, $G$, $C$ and $D$:
\begin{equation}
\min_{C, G}\max_{A, D} \big[ \mathcal{L}^{C}_{nat} + \lambda \mathcal{L}^{C,A}_{adv} + \gamma \mathcal{L}^{G,C,D}_{gan} \big], 
\label{Eq:loss}
\end{equation}
where $\lambda$ and $\gamma$ are weight factors. Basically, Equation (\ref{Eq:loss}) integrates the adversarial generation of UAEs and the AT over UAEs within a unified min-max game. Algorithm \ref{Alg:UAT} gives the sketch of the joint training, where the Lines 6 and 7 solve the inner maximization with stochastic gradient ascent, while the Lines 8 and 9 solve the outer minimization with stochastic gradient descent. 

\renewcommand{\algorithmicrequire}{\textbf{Input:}}
\renewcommand{\algorithmicensure}{\textbf{Output:}}

\begin{algorithm}[!t]
\caption{Training Algorithm of PUAT}
\label{Alg:UAT}
\begin{algorithmic}[1]
\REQUIRE ~~ 
Semi-supervised datasets $\mathcal{D}$;
 \ENSURE ~~ 
 Well-trained target classifier $C$;

\STATE Random initialize $A$, $C$, $G$, and $D$; 
\STATE Pre-trained $C$, $G$, and $D$ for $T_{pre}$ epochs; 
\FOR {$T$ epochs}
	\STATE Generate UAEs $\{(\tilde{x}_g,y)\}$ using Equations (\ref{Eq:z_a}) and (\ref{Eq:tilde}).
	\STATE Generate pseudo-data pairs $\{ (x_g, y)\}$ w.r.t. Equation (\ref{Eq:LG}).
	\STATE Generate pseudo-label pairs $\{ (x, y_c) \}$ w.r.t. Equation (\ref{Eq:LC}).
	\STATE Update $D$ by \textbf{ascending} stochastic gradient \\ $\nabla_{D}(\mathcal{L}^{C,D}_{gan}$ $+\mathcal{L}^{G,D}_{gan})$;
	\STATE Update $A$ by \textbf{ascending} stochastic gradient \\ $\nabla_{A} (\mathcal{L}^{C,A}_{adv})$;
	\STATE Update $C$ by \textbf{descending} stochastic gradient \\ $\nabla_{C}( \mathcal{L}^C_{nat}+\lambda \mathcal{L}^{C,A}_{adv} + \gamma \mathcal{L}^{C,D}_{gan})$ and use $C$ to update $C^\prime$;
	\STATE Update $G$ by \textbf{descending} stochastic gradient \\ $\nabla_{G}(\mathcal{L}^{G,D}_{gan})$;

\ENDFOR
\end{algorithmic}
\end{algorithm}

\section{Theoretical Justification}
\label{sec:prove}

In this section, we will theoretically justify: 1) PUAT's ability to simultaneously improve adversarial robustness and standard generalizability of the target classifier without compromising either of them; 2) superiority of UAE to RAE. We first discuss the ideal setting with infinite labeled data available for training, and then the practical setting with finite labeled data.

\subsection{
Infinite Data
}
\label{Sec:infinite}
In this section, we discuss the population loss which is defined over the overall distribution $P(x,y)$, i.e., dataset $ \mathcal{D}_l $ $= \mathcal{O}$. 
For the brevity of later derivations, in the following text $P(x, y)$, $P_G(x, y)$, $P_T(x, y)$ and $P_C(x, y)$ will be abbreviated to $P$, $P_G$, $P_T$ and $P_C$, respectively, when the context is unambiguous. 

Recall that PUAT offers the adversarial robustness by $P_C = P_T$, while the standard generalizability by $P_C = P$. As Equation (\ref{Eq:loss}) defines a unified framework integrating the UAE generation with the AT against UAEs, proving PUAT's ability to eliminate the tradeoff is equivalent to proving that the optimal solution of Equation (\ref{Eq:AT}) will be achieved without any conflict, which means the adversarial robustness and the standard generalizability agree with each other under the framework of Equation (\ref{Eq:loss}). 
We first give the equilibrium of the G-C-D GAN by the following lemma:
\begin{lemma}
The optimal solution of $\min_{C, G} \max_{D} \mathcal{L}^{G,C,D}_{gan}$ (Equation (\ref{Eq:L_gan})) is achieved at $P(x,y)=(P_G (x, y) + P_C (x, y)) / 2$.
\label{Lemma:GCD}
\end{lemma}
The proof of Lemma \ref{Lemma:GCD} can be found in Appendix A. Now we show that $\min_{C}\mathcal{L}^C_{nat}$ and $\min_{C}$$\max_{A}$ $\mathcal{L}^{C,A}_{adv} $ are two consistent optimization problems with the same optimal solution $C^*$. 
\begin{theorem}
In Equation (\ref{Eq:AT}), $\min_{C}\mathcal{L}^C_{nat}$ and $\min_{C}$$\max_{A}$ $\mathcal{L}^{C,A}_{adv} $ can be achieved at the same optimal $C^*$ subjected to $P_{G^*}=P_{C^*}=P$. 
\label{Theorem:GCD}
\end{theorem}
\begin{proof}
(1) By the definition of $\mathcal{L}^C_{nat}$ (Equation (\ref{Eq:sup_loss})), the optimal $C$ for $\min_C \mathcal{L}^C_{nat}$ is
\begin{equation*}
\begin{aligned}
C^*_{nat}=
& \argmin_C  \mathcal{L}^{C}_{nat} \\ 
= & \argmin_C \iint P(x, y) \log \frac{P(x) P(x, y)}{P_C(x,y) P(x, y)} \text{ dxdy} \\
= & \argmin_C \iint P(x, y) \log \frac{P(x, y)}{P_C(x, y)} \text{ dxdy} + H(y | x)  \\
= & \argmin_C \text{KL}\big(P \Vert P_C\big),
\end{aligned}
\end{equation*}
where KL is Kullback–Leibler divergence, $H(y | x)$ is the conditional entropy that independent of $C$, and the second equality holds because $P_C(x, y) = P(x) P_C(y|x) $.  Since many works have proved that consistency loss can help classification, we simply ignore the consistency term here, i.e., we take $\alpha=0$.  
Therefore, $C^*_{sup}=\argmin_C \mathcal{L}^C_{nat}$ is achieved at the KL divergence equals zero, 
\begin{equation*}
P_{C^*_{nat}} = P.
\end{equation*}

(2) Let $A^*=\argmax_{A}  \mathcal{L}^{C,A}_{adv}$ be the optimal $A$ and then fix it, 
we first analysis $C$ according to Equation (\ref{Eq:adver_loss}), 

\begin{equation*}
\begin{aligned}
C^*_{adv}
= &{}\argmin_{C}  \mathcal{L}^{C,A^*}_{adv} \\
= &{}\argmin_{C}  \iint P_{T}(x, y) \log \frac{P(x) P_T(x, y)}{P_{C}(x,y) P_T(x, y)} \text{ dxdy} \\
= &{}\argmin_{C}  \iint P_{T}(x, y) \log \frac{ P_{T}(x, y) } { P_C(x, y) } \text{ dxdy} \\
= &{}\argmin_{C}  \text{KL} \big(P_{T} \Vert P_C\big),  
\end{aligned}
\end{equation*}
in which $P_{T}$ is the perturbed sample distribution. Obviously, $\text{KL} \big(P_{T} \Vert P_{C} \big) \ge \text{KL} \big(P_{G}\Vert P_{C} \big)\ge 0$, which indicates that $\min_{C}\mathcal{L}^{C,A^*}_{adv} $ essentially minimizes the upper-bound of the KL-divergence of $ P_{G}$ and $ P_{C}$. 
Therefore, $C^*_{adv}=\argmin_C \mathcal{L}^{C,A^*}_{adv} $ can be achieved at  
\begin{equation*}
P_{C^*_{adv}}=P_{G}.
\end{equation*}

(3) The part (1) and part (2) together tell us that if $P$ and $P_G$ are two different distributions, the standard generalizability and adversarial robustness of the target classifier $C$ will be at odd, i.e., $C^*_{sup}\neq C^*_{adv}$, and $C$ will converge to a solution that depends on the tradeoff. But conversely, if we simultaneously adjust $G$ by G-C-D GAN (Equation (\ref{Eq:L_gan})), the tradeoff will be eliminated. Combining Lemma \ref{Lemma:GCD} leads to the same result regardless of $C^*_{nat}$ or $C^*_{adv}$, 
\begin{equation*}
P_{G^*}=2P-P_{C^*}=P.
\end{equation*}
equationTherefore, minimizing $\mathcal{L}^C_{nat}$ and $\max_{A}$ $\mathcal{L}^{C,A}_{adv} $ for $C$ are consistent optimization problems with the same optimal solution $C^*$ subjected to $P_{C^*}=P=P_{G^*}$. \end{proof}
\begin{remark}
Theorem \ref{Theorem:GCD} shows that $ \mathcal{L}^{C,A}_{adv}$ has no conflict with $\mathcal{L}^C_{nat}$, which further implies adversarial robustness is not necessarily at odd with standard generalization, and our PUAT can achieve adversarial robustness without compromising standard accuracy. 
\end{remark}
Theorem \ref{Theorem:GCD} tells us why PUAT can consistently optimize the standard generalization by applying UAEs to the AT. Now we further compare the adversary of UAE and RAE. When the labeled data is infinite, the adversary of RAEs is defined by the target classifier's loss they incurred over $P$, i.e., 
\begin{equation*}
\mathbb{E}_{(x,y)\sim P} \max_{\hat{x}: \Vert \hat{x} - x\Vert_p \leq \epsilon}l(\hat{x},y;C),
\end{equation*}
where $\epsilon$ is perturbation budget, and $\Vert \cdot \Vert_p$ is $p$-norm. According to Equation (\ref{Eq:Ladv_}), the adversary of the UAEs satisfying the same perturbation budget is 
\begin{equation*}
\max_{T: \Vert T(x) - x\Vert_p \leq \epsilon} \mathbb{E}_{(x ,y)\sim P} l(T(x),y;C).
\end{equation*}
We have the following theorem:
\begin{theorem} 
For a target classifier $C$ with loss function $l$, the UAE adversary is equivalent to the RAE adversary under the same perturbation budget, i.e.,
\begin{equation*}
\begin{aligned}
& \max_{T: \Vert T(x) - x\Vert_p \leq \epsilon} \mathbb{E}_{(x ,y)\sim P} l(T(x),y;C) \\
= & \mathbb{E}_{(x,y)\sim P} \max_{\hat{x}: \Vert \hat{x} - x\Vert_p \leq \epsilon}l(\hat{x},y;C). 
\end{aligned}
\end{equation*}
\label{Theorem:UAE_adversary}
\end{theorem}
\begin{proof}
For simplicity, the constrains $\Vert \tilde{x} - x\Vert_p \leq \epsilon$ and $\Vert $ $T(x) $ $- x\Vert_p $ $\leq \epsilon$ are omitted in this proof hereinafter. At first, one can note that proving Theorem \ref{Theorem:UAE_adversary} is equivalent to proving the proposition that if $T^*=\argmax_{T} \mathbb{E}_{(x ,y)\sim P} l(T(x),y;C)$, then $\forall (x,y)\sim P$, 
\begin{equation*}
l(T^*(x),y;C)=\max_{\hat{x}} l(\hat{x},y;C).
\end{equation*}
Equivalently, we prove its contrapositive, i.e., if there exists a sample $(x_k,y_k)\sim P$ making $l(T_1(x_k),y_k;C)$$<$$\max_{\hat{x}_k}$ $ l(\hat{x}_k,y_k;C)$, then $T_1 $$ \neq $$ \argmax_{T} \mathbb{E}_{(x ,y)\sim P} l(T(x),y;C)\text{ , i.e.}$, $\mathbb{E}_{(x ,y)\sim P} l(T_1(x),y;C) $$<$$ \mathbb{E}_{(x ,y)\sim P} l(T^*(x),y;C)$. Note that it is impossible that $l(T_1(x_k), y_k ;C)$$>$$\max_{\hat{x}_k}$$l(\hat{x}_k, y_k ;C)$ because $T_1(x_k)$ and $\hat{x}_k$ are AEs in the same neighborhood of $x_k$ and $\max_{\hat{x}_k}l(\hat{x}_k,y_k;C)$ is the maximal loss incurred by the adversarial examples in this neighborhood.
\begin{equation*}
\begin{aligned}
& \mathbb{E}_{(x ,y)\sim P} l(T_1(x),y;C) \\
= & \sum_{i \ne k} P(x_i, y_i) l(T_1(x_i),y_i;C) + P(x_k, y_k) l(T_1(x_k),y_k;C).
\end{aligned}
\end{equation*}
Since for $\forall i $$\ne$$ k$, $ l(T_1(x_i),y_i;C) $$\le $$\max_{\hat{x}_i}l(\hat{x}_i,y;C)$, and $l(T_1(x_k),y;C)$$<$$\max_{\hat{x}_k}$ $ l(\hat{x}_k,y;C)$, we have
\begin{equation*}
\begin{aligned}
\mathbb{E}_{(x ,y)\sim P} l(T_1(x),y;C) < & \sum_{i} P(x_i, y_i) \max_{\hat{x}_i} l(\hat{x}_i,y_i;C) \\
= &\mathbb{E}_{(x,y)\sim P} \max_{\hat{x}} l(\hat{x},y;C). 
\end{aligned}
\end{equation*}
Let another mapping $T_2(x)=\argmax_{\hat{x}}l(\hat{x},y;C)$, and then $\mathbb{E}_{(x,y)\sim P} \max_{\hat{x}} l(\hat{x},y;C) = \mathbb{E}_{(x ,y)\sim P} l(T_2(x),y;C)$. Hence
\begin{equation*}
\begin{aligned}
\mathbb{E}_{(x ,y)\sim P} l(T_1(x),y;C) < \mathbb{E}_{(x ,y)\sim P} l(T_2(x),y;C), 
\end{aligned}
\end{equation*} 
and thus $T_1 \neq \argmax_{T} \mathbb{E}_{(x ,y)\sim P} l(T(x),y;C)$. 
\end{proof}
\begin{remark} 
Theorem \ref{Theorem:UAE_adversary} demonstrates that why the AT over UAEs can offer the adversarial robustness against RAEs. In summary, Theorem \ref{Theorem:GCD} and Theorem \ref{Theorem:UAE_adversary} explain why PUAT can consistently optimize the standard generalization and adversarial robustness. These together form the theoretical foundation of the tradeoff improvement by PUAT. 
\label{Remark:UAE_adversary}
\end{remark}

\subsection{Finite Data}
\label{Sec:finite}
Now we give the theoretical results of PUAT in the practical case where data is finite. Specifically, we will establish the quantitative relationship between the generalizability of PUAT and the amount of the semi-supervised training dataset $\vert \mathcal{D}_c \cup \mathcal{D}_l \vert$. At first, we give the generalization error bound of G-C-D GAN when it is optimized over finite data.

\subsubsection{The Generalization Error Bound of G-C-D GAN}
The following lemma shows when the G-C-D GAN is optimized over finite data (i.e., using the empirical loss $\hat{\mathcal{L}}^{G,C,D}_{gan} = \hat{\mathcal{L}}_{D} + \frac{1}{2} \hat{\mathcal{L}}_{C} +  \frac{1}{2} \hat{\mathcal{L}}_{G}$), the relationship between the generalization error of G-C-D GAN, i.e., how close the distribution $P_{GC}$ learned by $G$ and $C$ is to $P$ (which is measured by $\text{TV}$ $(P, P_{GC})$), and the amount of data.
\begin{lemma}
Let $\mathcal{Z}$ be the set of sampled noise  $z\sim P_z(z)$, $m$ $=\vert \mathcal{D}_l \vert$, $n=\vert \mathcal{D}_c \vert$, and $b = \Vert \log \frac{P}{P_{GC}} + 1 \Vert_\infty$ $ = \max_{(x, y)}$ $\log$ $ \frac{P}{P_{GC}} + 1$. For any $G$, $C$ and $0< \delta <1$, if $\vert\mathcal{Z}\vert \to \infty$, $\text{TV}(P, P_{GC}) \leq B$ holds with probability at least $(1 - \delta)^2$, where $B = \Big( \frac{1}{2}b
+\frac{1}{4} \max_{D} \hat{\mathcal{L}}^{G,C,D}_{gan} 
+ \sqrt{\frac{\log{\frac{1}{\delta}}}{8m}} + \sqrt{\frac{ \log{\frac{1}{\delta}}}{32n}} \Big)^{\frac{1}{2}}$.
\label{Lemma:GCD_finite}
\end{lemma}
The proof of Lemma \ref{Lemma:GCD_finite} can be found in Appendix A. It is noteworthy that in Lemma \ref{Lemma:GCD_finite} the condition $\vert\mathcal{Z}\vert \to \infty$ is reasonable since we can always generate arbitrarily many $z \sim P_{z}(z)$. This means that when $\vert\mathcal{Z}\vert$ is large enough, its impact on the generalization error $\text{TV}(P,P_{GC})$ can be ignored. Specifically, if $\vert\mathcal{Z}\vert \to \infty$, then $\hat{\mathcal{L}}_{G} \to \mathcal{L}_{G}$, allowing us to focus on the impact of $m+n$, the amount  of the semi-supervised training data. Therefore, Lemma \ref{Lemma:GCD_finite} establishes the bound of $\text{TV}(P,P_{GC})$ on finite data, and shows that as $m$ and $n$ increase, the optimization over finite data with $\min_{G,C} \max_{D} \hat{\mathcal{L}}^{G,C,D}_{gan}$ as objective will lead to smaller $\text{TV}(P,P_{GC})$. Consequently $P_{GC}$ gradually approaches $P$, which is consistent with the optimization result over infinite data claimed by Lemma \ref{Lemma:GCD}.

\subsubsection{The Generalization Error Bound of PUAT}
Based on Lemma \ref{Lemma:GCD_finite}, now we can give the generalization error bound of PUAT when it is optimized on finite data, which is claimed by the following theorem. Recall that the generalization of PUAT includes two parts, the generalization over natural samples (i.e., the standard generalizability), which is achieved by minimizing $\hat{\mathcal{L}}^{C}_{nat}$, and the generalization over adversarial examples (i.e., the adversarial robustness), which is achieved by minimizing $\max_{A} \hat{\mathcal{L}}^{C,A}_{adv}$. The following theorem confirms that on finite data, these two generalizations are also consistent, and gives their generalization error bounds. Again, in the following text we use $P$, $P_G$, and $P_C$ to represent $P(x, y)$, $P_G(x, y)$, and $P_C(x, y)$, respectively, when the context is unambiguous.
\begin{theorem} 
Let $m=\vert \mathcal{D}_l \vert $, $ n=\vert \mathcal{D}_c \vert$. The following results hold:
\newline
(1) $\text{TV}(P, P_C) \leq \frac{1}{\sqrt{2}} \big( \hat{\mathcal{L}}^{C}_{nat} + b_1 \sqrt{\frac{\log{\frac{1}{\delta}}}{2m}} - R^*_1 \big)^{\frac{1}{2}} $ with probability at least $1-\delta$ for any $0<\delta<1$, where $b_1 = \left\Vert -\log{P_C(y \vert x)} \right\Vert_\infty$, and $R^*_1 = - \mathbb{E}_{P(x, y)} \log P(y \vert x)$ is Bayes error.
\newline
(2) $\text{TV}(P, Q) \leq \frac{1}{2\sqrt{2}} \big( \max_{A} \hat{\mathcal{L}}^{C,A}_{adv} - R^*_2 \big)^{\frac{1}{2}}  + B$ with probability at least $(1-\delta)^2$ for any $0<\delta<1$ if $\vert \mathcal{Z} \vert \to \infty$, where $Q\in \{P_C, P_G\}$, and $R^*_2 = - \mathbb{E}_{P_G} \log $ $\frac{P_G(x,y)}{P(x)}$ is Bayes error. 
\label{Theorem:GCD_finite}
\end{theorem}
\begin{proof}
(1) Note that for the population loss $\mathcal{L}^{C}_{nat}$ defined by the first line of Equation (\ref{Eq:sup_loss}), we have 
\begin{equation*}
\begin{aligned}
\mathcal{L}^{C}_{nat} =& \iint P(x, y) \log \left(\frac{P(x, y)}{P_C(y\vert x) P(x)} \cdot \frac{1}{P(y \vert x)} \right) \text{ dxdy} \\
=& \iint P(x, y) \log \frac{P(x, y)}{P_C(x, y)} \text{ dxdy} - \mathbb{E}_{P(x, y)} \log P(y \vert x)\\
=& \text{KL}\big(P \Vert P_C\big) + R^*_1 ,
\end{aligned}
\end{equation*}
According to Hoeffding's Inequality \cite{hoeffding1994probability}, with probability at least $1-\delta$ the following inequality holds: 
\begin{equation*}
\begin{aligned}
\mathcal{L}^{C}_{nat} \leq \hat{\mathcal{L}}^{C}_{nat} + b_1 \sqrt{\frac{ \log{\frac{1}{\delta}}}{2m}} ,
\end{aligned}
\end{equation*}
where $b_1 = \left\Vert -\log{P_C(y \vert x)} \right\Vert_\infty$ is the upper bound of the loss on a single sample $(x,y) \sim P$. Then by applying Pinsker’s Inequality \cite{cover1999elements} as we do in the proof of Lemma \ref{Lemma:GCD_finite}, it is easy to show that the following inequality holds with probability at least $1-\delta$: 
\begin{equation*}
\begin{aligned}
\text{TV}(P, P_C) & \leq \sqrt{  \frac{1}{2} \text{KL}\big(P \Vert P_C\big)} = \sqrt{  \frac{1}{2} (\mathcal{L}^{C}_{nat} - R^*_1)} \\
&\leq \frac{1}{\sqrt{ 2}} \Big(\hat{\mathcal{L}}^{C}_{nat} + b_1 \sqrt{\frac{\log{\frac{1}{\delta}}}{2m}} - R^*_1 \Big)^{\frac{1}{2}} . 
\end{aligned}
\end{equation*}
(2) In this part, we first (a) obtain the upper bound of $\text{KL} (P_G \Vert P_C)$ through $\mathcal{L}^{C,A}_{adv}$, and then (b) obtain the upper bound of $\text{TV} (P, P_{C})$ and $\text{TV} (P, P_{G})$ using $\text{KL} (P_G \Vert P_C)$ and $\text{TV} (P, P_{GC})$.

(a) According to Equation (\ref{Eq:adver_loss}),   
\begin{equation*}
\begin{aligned}
\mathcal{L}^{C,A}_{adv} 
& = \mathbb{E}_{P_T (x,y)} \big[-\log P_C(y|x) \big]\\
& \geq \mathbb{E}_{P_G (x,y)} \big[-\log P_C(y|x) \big]\\
& = \iint P_G(x,y) \log \left(\frac{P_G(x, y) }{P_C(y \vert x) P(x) } \cdot \frac{P(x) }{ P_G(x, y)}  \right) \text{ dxdy} \\
& = \text{KL}(P_G \Vert P_C) + R^*_2.
\end{aligned}
\end{equation*}
Hence
\begin{equation*}
\begin{aligned}
\text{KL}(P_G \Vert P_C) \leq \mathcal{L}^{C,A}_{adv} - R^*_2.
\end{aligned}
\end{equation*}
According to PAC learning theory \cite{shalev2014understanding}, with probability at least $1-\delta$ the following inequality holds: 
\begin{equation*}
\begin{aligned}
\mathcal{L}^{C,A}_{adv} \leq \hat{\mathcal{L}}^{C,A}_{adv} + b_2 \sqrt{\frac{log{\frac{1}{\delta}}}{2m \vert \mathcal{Z} \vert}}, 
\end{aligned}
\end{equation*}
where $b_2 = \left\Vert -\log{P_C(y \vert \tilde{x})} \right\Vert_\infty$ is the upper bound of the loss on a single sample from $P_T$. Since $\lim_{\vert \mathcal{Z} \vert \to \infty} \mathcal{L}^{C,A}_{adv} = \hat{\mathcal{L}}^{C,A}_{adv}$, taking $A^* = \argmax \hat{\mathcal{L}}^{C,A}_{adv}$, we can get the following inequality: 
\begin{equation*}
\begin{aligned}
\text{KL}\big(P_G \Vert P_C\big) \leq \max_{A} \hat{\mathcal{L}}^{C,A}_{adv} - R^*_2. 
\end{aligned}
\end{equation*}

(b) Since 
\begin{equation*}
\begin{aligned}
\text{TV}(P, P_C) = &\frac{1}{2}  \iint \vert P_C - P \vert \text{ dxdy} \\
= &\frac{1}{2} \iint \left\vert \frac{1}{2}( P_C - P_G ) + \frac{1}{2}( P_C + P_G ) - P \right\vert \text{ dxdy} \\
\leq & \frac{1}{4} \iint \vert P_C - P_G \vert \text{ dxdy} + \frac{1}{2} \iint \vert P_{GC} - P \vert \text{ dxdy}\\
= & \frac{1}{2} \text{TV} (P_G, P_C) + \text{TV} (P, P_{GC}) 
\end{aligned}
\end{equation*}
and 
\begin{equation*}
\begin{aligned}
\text{TV}(P, P_G) = &\frac{1}{2} \iint \vert P_G - P \vert \text{ dxdy} \\
= & \frac{1}{2} \iint \left\vert \frac{1}{2}( P_G - P_C ) + \frac{1}{2}( P_G + P_C ) - P \right\vert \text{ dxdy} \\
\leq & \frac{1}{4} \iint \vert P_G - P_C \vert \text{ dxdy} + \frac{1}{2} \iint \vert P_{GC} - P \vert \text{ dxdy}\\
= & \frac{1}{2} \text{TV} (P_G, P_C) + \text{TV} (P, P_{GC}), 
\end{aligned}
\end{equation*}
$\text{TV} (P, P_{C})$ and $\text{TV} (P, P_{G})$ share the same upper bound. For any $Q \in \{P_C, P_G\}$, by Pinsker’s inequality and Lemma \ref{Lemma:GCD_finite}, 
\begin{equation*}
\begin{aligned}
\text{TV}(P, Q) &\leq \frac{1}{2} \text{TV} (P_G, P_C) + \text{TV} (P, P_{GC})\\ 
&\leq \frac{1}{2} \sqrt{\frac{1}{2} \text{KL} (P_G \Vert P_C)} + \text{TV} (P, P_{GC})\\
&\leq \frac{1}{2 \sqrt{2} } \Big( \max_{A} \hat{\mathcal{L}}^{C,A}_{adv} - R^*_2 \Big)^{\frac{1}{2}} + B
\end{aligned}
\end{equation*}
holds with probability at least $(1-\delta)^2$. 
\end{proof}
\begin{remark} 
The conclusion (1) of Theorem \ref{Theorem:GCD_finite} provides the generalization error gap incurred by optimizing over natural examples (i.e., minimizing $\hat{\mathcal{L}}^{C}_{nat}$), while the conclusion (2) gives the generalization error gap incurred by optimizing over adversarial examples (i.e., minimizing $\max_{A}\hat{\mathcal{L}}^{C,A}_{adv}$). We can see that minimizing $\hat{\mathcal{L}}^{C}_{nat}$ and $\max_{A}\hat{\mathcal{L}}^{C,A}_{adv}$ can consistently leads to smaller $\text{TV} (P, P_{C})$. This shows that even if the data is finite, PUAT can improve standard generalizability and adversarial robustness without conflict, which is consistent with the case of infinite data claimed by Theorem \ref{Theorem:GCD}. Additionally, Theorem \ref{Theorem:GCD_finite} indicates that as the amount of data increases (i.e., larger values of $m$ or $n$), PUAT can achieve better generalization (i.e., smaller values of $\text{TV} (P, P_{C})$)).
\label{Remark:GCD_finite}
\end{remark}

\subsubsection{
The Bound of the UAE Adversary Gap
}
\label{sec:uae-bound}
As demonstrated by Theorem \ref{Theorem:GCD_finite}, when data is finite, there exists generalization error between the learned distribution $P_G$ from which UAEs are drawn and the overall distribution $P$. Essentially, the difference between $P_G$ and $P$ is caused by the samples $(x, y)$ whose probabilities are different in $P_G$ and $P$, i.e., $P(x, y) \ne P_G(x, y)$. This discrepancy results in the UAE adversary gap, which is the difference between the adversary offered by the empirical UAEs drawn from $P_G$, and the ideal adversary offered by the UAEs drawn from $P$. For simplicity, we will omit the constrains $\Vert \tilde{x} - x\Vert_p \leq \epsilon$ and $\Vert T(x) - x\Vert_p \leq \epsilon$ in this subsection. Following the idea of Theorem \ref{Theorem:UAE_adversary}, the adversary offered by empirical UAEs is measured by $\mathbb{E}_{ (x, y) \sim P_G} \max_{\tilde{x}} l(\tilde{x},y;C)$, while the ideal adversary of UAEs is measured by $\mathbb{E}_{ (x, y) \sim P} \max_{\tilde{x}} l(\tilde{x},y;C)$. Then the adversary gap is 
\begin{equation*}
\begin{aligned}
G_U = \left\vert \mathbb{E}_{ (x, y) \sim P_G} \max_{\tilde{x}} l(\tilde{x},y;C) - \mathbb{E}_{(x, y) \sim P} \max_{\tilde{x}}l(\tilde{x},y;C) \right\vert.
\end{aligned}
\end{equation*} 
Theorem \ref{Theorem:UAE_adversary_finite} states that this gap is bounded by the number of the semi-supervised data.
\begin{theorem} 
Let $m=\vert \mathcal{D}_l \vert $, $ n=\vert \mathcal{D}_c \vert$. For any $0 < \delta < 1$, 
\begin{equation*}
\begin{aligned}
G_U \leq 2 B_1 \Big( B + \frac{1}{2\sqrt{2}} \big( \max_{A} \hat{\mathcal{L}}^{C,A}_{adv} - R^*_2 \big)^{\frac{1}{2}} \Big)
\end{aligned}
\end{equation*} 
holds with probability at least $(1-\delta)^2$, where $B$ and $R^*_2$ are defined in Lemma \ref{Lemma:GCD_finite} and Theorem \ref{Theorem:GCD_finite}, respectively. Furthermore, $B_1 = \max_{P(x, y) \ne P_G(x, y)}\max_{\tilde{x}} l(\tilde{x},y;C)$ measures the adversary of the AEs generated by the examples $(x, y)$ such that their probabilities differ in distributions $P$ and $P_G$. 
\label{Theorem:UAE_adversary_finite}
\end{theorem}
\begin{proof}
According to the concept of expectation and the property of integral, 
\begin{equation*}
\begin{aligned}
G_U = & \left\vert \iint \left(P_G(x, y) - P(x, y) \right) \max_{\tilde{x}} l(\tilde{x},y;C)  \text{ dxdy} \right\vert \\
\leq & \iint \left\vert P_G(x, y) - P(x, y) \right\vert  \max_{\tilde{x}} l(\tilde{x},y;C) \text{ dxdy}.
\end{aligned}
\end{equation*} 
Noting that only $(x, y)$ such that $P_G(x, y) \ne P(x, y)$ contributes to the integral,
\begin{equation*}
\begin{aligned}
G_U \leq  2 B_1 \text{TV} (P,P_G),
\end{aligned}
\end{equation*} 
where $\text{TV} (P,P_G) = \frac{1}{2}\iint \left\vert P_G(x, y) - P(x, y) \right\vert \text{ dxdy}$. According to the conclusion (2) of Theorem \ref{Theorem:GCD_finite},
\begin{equation*}
\begin{aligned}
G_U \leq & 2 B_1 \bigg( B +  \frac{1}{2\sqrt{2}} \Big( \max_{A} \hat{\mathcal{L}}^{C,A}_{adv} - R^*_2 \Big)^{\frac{1}{2}} \bigg)
\end{aligned}
\end{equation*} 
holds with probability at least $(1-\delta)^2$. 
\end{proof}
\begin{remark} 
It is easy to see that the upper bound of $G_U$ is a monotonically decreasing function of $m$ and $n$. Therefore, Theorem \ref{Theorem:UAE_adversary_finite} shows that an increase in the number of labeled or unlabeled data points reduces $G_U$ and results in a stronger adversary of empirical UAEs. This allows PUAT to offer better comprehensive robustness.
\label{Remark:UAE_adversary_finite}
\end{remark}
When $P_G \neq P$, there may exist natural samples $x\in supp(P)$ but $x \notin supp(P_G)$, which leads to some RAEs that are not covered by the learned UAE distribution. To rectify this and further minimize the adversary gap of the empirical UAE, one solution is to train PUAT using both UAEs and RAEs with the following loss function
\begin{equation}
\min_{C, G}\max_{A, D} \big[ \mathcal{L}^{C}_{nat} + \lambda \mathcal{L}^{C,A}_{adv} + \gamma \mathcal{L}^{G,C,D}_{gan} +  \beta \mathcal{L}^{C}_{r} \big], 
\label{Eq:loss_new}
\end{equation}
where $\mathcal{L}^{C}_{r}$ is the adversarial loss on RAE and $\beta$ controls its weight. Specifically, $\mathcal{L}^{C}_{r}$ is defined as 
\begin{equation}
\begin{aligned}
\mathcal{L}^{C}_{r} = \mathbb{E}_{(x, y) \sim P} \big[-\log P_C(y|\hat{x}) \big], 
\end{aligned}
\label{Eq:rae_loss}
\end{equation}
where $\hat{x}$ is an RAE satisfying a norm-constraint, which is the result by perturbing a natural sample $x$, i.e. $\hat{x} = \argmax_{\hat{x}:\Vert \hat{x}-x \Vert_p \leq \epsilon} l(\hat{x},y;C)$.

\section{Experiments}
\label{sec:experiments}
The goal of experiments is to answer the following research questions:
\begin{itemize}
\item \textbf{RQ1} How does PUAT perform as compared to the state-of-the-art baselines in terms of the standard generalization and comprehensive adversarial robustness?

\item \textbf{RQ2} Does PUAT improve the tradeoff between standard generalization and adversarial robustness?

\item \textbf{RQ3} How do UAEs and unlabeled data influence the performance of PUAT?

\item \textbf{RQ4} How to visually show the superiority of PUAT?

\item \textbf{RQ5} How long is the training of PUAT?

\end{itemize}

\subsection{Experimental Setting}
\subsubsection{Datasets}
We conduct the experiments on Tiny ImageNet \cite{le2015tiny}, ImageNet32 \cite{chrabaszcz2017downsampled}, SVHN \cite{netzerreading}, CIFAR10 \cite{krizhevsky2009learning}, and CIFAR100 \cite{krizhevsky2009learning}, which are widely used benchmarks for evaluating adversarial training. Tiny ImageNet has 200 classes each of which consists of 500 images for training and 50 ones for testing. Following \cite{Li2022}, for Tiny ImageNet we select 10 classes from 200 classes and downscale their resolution to $32 \times 32$. And for ImageNet32 which has 1,000 classes and 1,281,167 images, we use the subset consisting of its first 10 classes. SVHN is a set of 73,257 and 26,032 digit images for training and testing respectively. CIFAR10 contains 50,000 training images and 10,000 testing images distributed across 10 classes, while CIFAR100 has 50,000 training images and 10,000 testing images over 100 classes. 

Each training set is randomly divided into two parts, labeled data and unlabeled data, where unlabeled data are built by removing the labels of the images. We will repeat the random division of the dataset three times, and report the average results with standard deviation. Table \ref{Tb:datasets} shows the dataset configurations following the widely used semi-supervised settings \cite{tarvainen2017mean, Li2022}. Besides, on each dataset, we randomly select 20 percent of the labeled training data as validation set for the tuning of hyper-parameters.


%
%

\begin{table}[t]
	\centering
	\caption{
	Configuration of the datasets. 
	}
	\label{Tb:datasets}
	\begin{tabular}{l|c|c|c|c}
	\toprule
	\multirow{3}{*}{Datasets}&\multicolumn{2}{c|}{training}&\multirow{3}{*}{testing}&\multirow{3}{*}{classes}\\
	\cmidrule{2-3}
	&{labeled}&{unlabeled}&\\
	\midrule  
	{Tiny ImageNet\cite{le2015tiny}}&{1,000}&{4,000}&{500}&{10}\\
        {ImageNet32\cite{le2015tiny}}&{1,000}&{11,850}&{500}&{10}\\
	{SVHN\cite{netzerreading}}&{1,000}&{72,257}&{26,032}&{10}\\
	{CIFAR10\cite{krizhevsky2009learning}}&{4,000}&{46,000}&{10,000}&{10}\\
	{CIFAR100\cite{krizhevsky2009learning}}&{10,000}&{40,000}&{10,000}&{100}\\
	\bottomrule
	\end{tabular}
\end{table}

\begin{table}[t]
	\centering
	\caption{
	Characteristics of the baseline methods and PUAT.
	}
	\label{Tb:baselines}
	\setlength{\tabcolsep}{0.17cm}
	\begin{tabular}{l|c|c|c|c}
	\toprule
	{Methods}&{RAE}&\makecell[c]{UAE}&{\makecell[c]{unlabeled data}}&{\makecell[c]{generated data}}\\
	\midrule
	{Regular}& & & &\\
	{TRADES \cite{zhang2019theoretically}}&{$\checkmark$}& & & \\
	{DMAT \cite{wang2023better}}&{$\checkmark$}& & &{$\checkmark$} \\
	{RST \cite{carmon2019unlabeled}}&{$\checkmark$}& & {$\checkmark$}&\\
	{PUAT}&{$\checkmark$}&{$\checkmark$} &{$\checkmark$}&{$\checkmark$}\\	
	\bottomrule
	\end{tabular}
\end{table}

\subsubsection{Baseline Methods}
\label{Sec:baselines}
To verify the superiority of PUAT, we compare it with the following baseline methods, whose characteristics are summarized in Table \ref{Tb:baselines}.
\begin {itemize}
\item \textbf{Regular} Regular method trains the target classifier using labaled natural data with cross-entropy loss, which plays the role of performance benchmark.

\item \textbf{TRADES} \cite{zhang2019theoretically} TRADES is an AT method based on RAEs that are generated by PGD based on labeled data, which can trade adversarial robustness off against standard generalizability by optimizing a regularized surrogate loss. 

\item\textbf{DMAT} \cite{wang2023better} DMAT performs an RAE-based AT on a training dataset augmented by samples generated by a class conditional elucidating diffusion model (EDM) \cite{karras2022elucidating}. 

\item\textbf{RST} \cite{carmon2019unlabeled} RST is also an AT method based on RAEs that are generated by PGD. Different from TRADES, RST first trains a classifier to predict labels for unlabeled data and then feeds these pseudo-labeled data together with labeled data into the sequel adversarial training algorithm. 

\end {itemize}

\subsubsection{Evaluation Protocol}
\label{Sec:evaluation}
We will use PUAT and the baseline methods to train the target classifier and compare their performances from the perspectives of the target classifier's standard generalizability and adversarial robustness. The generalizability is evaluated in terms of the target classifier's \textbf{natural accuracy} on testing natural examples, while the adversarial robustness is evaluated in terms of the target classifier's \textbf{robust accuracy} on the testing adversarial examples which are generated by various attack methods. 
In particular, for the robustness against RAE, we choose PGD \cite{madry2017towards} and Auto Attack (AA) \cite{croce2020reliable} to generate testing RAEs, which are the attacking methods widely adopted by the existing works. 
For the robustness against UAE which beyond RAE, we also choose GPGD \cite{xiao2022understanding} and USong \cite{song2018constructing} as the attacking methods, both of which are GAN-based methods and dedicated to the generation of UAE. 
The hyper-parameters of the attacking methods are shown in Appendix B, where $\epsilon$ is the perturbation budget and the bigger it, the stronger the attack. 
To evaluate the performance of PUAT and the baseline methods, we first invoke the attacking methods to generate testing AEs that can cause worst-case loss of the target classifier well trained by PUAT or a baseline method, and then check the classifier's accuracy on those testing AEs.

\subsubsection{Hyper-parameter Setting and Implement Details}
\label{Sec:implement}
All the hyper-parameters and the architecture of PUAT are shown in Appendix B, where the generator $G$, discriminator $D$ are implemented with respect to those adopted in \cite{li2017triple}, and $A$ is implemented as an MLP plus two residual blocks. 
During the training and testing, we employ the Wide ResNet (WRN-28-10) \cite{zagoruyko2016wide} as the target classifier $C$. 
We use SGD as the optimizer with Nesterov momentum \cite{nesterov1983method} and cosine annealed cyclic learning rate schedule \cite{smith2019super}, where learning-rate, weight-decay, and momentum are set to 0.2, $5\text{E-}4$, and 0.9, respectively. During training, a batch consists of 256 labeled images and 256 unlabeled ones, and early stopping is adopted as default option. 

To avoid gradient disappearance and mode collapse for the training of the G-C-D GAN, spectral norm \cite{miyato2018spectral} together with hinge GAN loss \cite{lim2017geometric, kavalerov2021multi} is used for $D$ to stabilize the training. When labeled data is severely sparse, it is difficult to train $G$ to precisely capture the real distribution $P(x, y)$. To overcome this issue, following the idea of \cite{tarvainen2017mean}, during the training we augment $\mathcal{D}_l$ with the pseudo-labeled pairs $\{ (x_c, \hat{y}_c) \}$, where $x_c \in \mathcal{D}_c$ and $\hat{y}_c = \argmax_{y_c \in \mathcal{Y}} P_C(y=y_c|x_c)$. Note that here the pseudo label $\hat{y}_c$ is a hard label represented by a one-hot vector, while in Equation (\ref{Eq:LC}) the pseudo label $y_c$ is a soft label represented by a softmax vector. This will not introduce conflict to the training of C-D GAN because (1) if $\hat{y}_c=y_c$, the loss incurred by $\{ (x_c, \hat{y}_c) \}$ and the loss incurred by $\{ (x_c, y_c) \}$ will cancel each other out; (2) if $\hat{y}_c \ne y_c$, $D$ will force $C$ to assign $\hat{y}_c$ a higher probability, which is similar to self-training on $C$.

\begin{table*}[t]
	\renewcommand{\arraystretch}{}
	\centering
	\caption{
	Performance on \textbf{Tiny ImageNet}.
	}
	\label{Tb:white_tiny}
	\setlength{\tabcolsep}{0.15cm}
	\begin{tabular}{l|c|c|c|c|c|c|c|c}
	\toprule
	\multirow{3}{*}{Methods}&\multirow{3}{*}{\makecell[c]{Natural \\Accuracy}}&\multicolumn{4}{c|}{Robust Accuracy to RAE}&\multicolumn{3}{c}{Robust Accuracy to UAE}\\
	\cmidrule{3-9}
	& &{PGD-2/255}&{PGD-4/255}&{PGD-8/255}&{AA}&{GPGD-0.01}&{GPGD-0.1}&{USong}\\
	\midrule
	{Regular}&{$65.93\pm0.25$}&{$22.13\pm2.76$}&{$2.33\pm0.52$}&{$0.0\pm0.0$}&{$0.0\pm0.0$}&{$33.27\pm0.84$}&{$14.60\pm0.75$}&{$14.40\pm1.56$}\\
	{TRADES}&{$58.33\pm1.32$}&{$51.20\pm0.28$}&{$44.53\pm0.77$}&{$31.47\pm2.31$}&{$29.47\pm1.61$}&{$40.47\pm0.66$}&{$20.27\pm0.90$}&{$43.53\pm2.29$}\\
	{DMAT}&{$50.90\pm3.10$}&{$45.00\pm2.80$}&{$37.90\pm2.90$}&{$26.60\pm1.60$}&{$22.00\pm1.60$}&{$37.70\pm3.30$}&{$22.70\pm0.90$}&{$43.30\pm2.30$}\\
	{RST}&{$43.87\pm3.17$}&{$38.80\pm2.14$}&{$33.20\pm1.34$}&{$24.00\pm1.51$}&{$21.60\pm2.73$}&{$30.93\pm3.47$}&{$18.13\pm2.32$}&{$35.60\pm2.26$}\\
	\midrule
	{PUAT}&\bm{$69.00\pm0.40$}&\bm{$62.40\pm0.20$}&\bm{$55.60\pm0.20$}&\bm{$38.60\pm0.40$}&\bm{$37.40\pm1.20$}&\bm{$44.90\pm2.10$}&\bm{$27.10\pm0.50$}&\bm{$48.80\pm2.20$}\\
	\bottomrule
	\end{tabular}
\end{table*}

\begin{table*}[t]
	\renewcommand{\arraystretch}{}
	\centering
	\caption{
	Performance on \textbf{ImageNet32}.
	}
	\label{Tb:white_imagenet32}
	\setlength{\tabcolsep}{0.15cm}
	\begin{tabular}{l|c|c|c|c|c|c|c|c}
	\toprule
	\multirow{3}{*}{Methods}&\multirow{3}{*}{\makecell[c]{Natural \\Accuracy}}&\multicolumn{4}{c|}{Robust Accuracy to RAE}&\multicolumn{3}{c}{Robust Accuracy to UAE}\\
	\cmidrule{3-9}
	& &{PGD-2/255}&{PGD-4/255}&{PGD-8/255}&{AA}&{GPGD-0.01}&{GPGD-0.1}&{USong}\\
	\midrule
	{Regular}&\bm{$53.13\pm1.33$}&{$18.87\pm1.65$}&{$6.87\pm0.77$}&{$1.40\pm0.28$}&{$1.33\pm0.38$}&{$34.53\pm2.45$}&{$22.20\pm1.14$}&{$24.13\pm2.31$}\\
	{TRADES}&{$32.80\pm1.28$}&{$28.53\pm1.24$}&{$23.53\pm2.23$}&{$17.47\pm2.78$}&{$15.60\pm2.27$}&{$32.13\pm0.69$}&{$20.53\pm1.05$}&{$30.47\pm0.98$}\\
	{RST}&{$48.07\pm2.81$}&{$41.27\pm1.75$}&{$35.40\pm1.47$}&{$26.47\pm0.66$}&{$22.60\pm1.28$}&{$39.80\pm1.25$}&{$23.07\pm0.68$}&{$43.80\pm1.89$}\\
	\midrule
	{PUAT}&{$51.20\pm0.71$}&\bm{$45.33\pm0.50$}&\bm{$39.87\pm0.25$}&\bm{$27.40\pm0.33$}&\bm{$25.93\pm0.41$}&\bm{$41.47\pm1.11$}&\bm{$25.33\pm1.79$}&\bm{$45.00\pm1.86$}\\
	\bottomrule
	\end{tabular}
\end{table*}

\begin{table*}[t]
	\renewcommand{\arraystretch}{}
	\centering
	\caption{
	Performance on \textbf{SVHN}.
	}
	\label{Tb:white_svhn}
	\setlength{\tabcolsep}{0.15cm}
	\begin{tabular}{l|c|c|c|c|c|c|c|c}
	\toprule
	\multirow{3}{*}{Methods}&\multirow{3}{*}{\makecell[c]{Natural \\Accuracy}}&\multicolumn{4}{c|}{Robust Accuracy to RAE}&\multicolumn{3}{c}{Robust Accuracy to UAE}\\
	\cmidrule{3-9}
	& &{PGD-2/255}&{PGD-4/255}&{PGD-8/255}&{AA}&{GPGD-0.01}&{GPGD-0.1}&{USong}\\
	\midrule
	{Regular}&{$73.76\pm0.63$}&{$4.73\pm2.35$}&{$0.15\pm0.11$}&{$0.0\pm0.0$}&{$0.0\pm0.0$}&{$78.27\pm1.11$}&{$48.42\pm2.40$}&{$70.10\pm3.01$}\\
	{TRADES}&{$58.35\pm2.14$}&{$44.67\pm1.71$}&{$32.44\pm1.12$}&{$15.71\pm0.66$}&{$13.31\pm0.58$}&{$59.27\pm2.88$}&{$32.64\pm1.24$}&{$61.73\pm1.63$}\\
	{DMAT}&{$83.65\pm0.40$}&{$80.71\pm1.41$}&{$74.49\pm1.57$}&{$55.41\pm1.71$}&{$47.26\pm1.62$}&{$81.40\pm1.27$}&{$58.90\pm2.39$}&{$93.91\pm2.11$}\\
	{RST}&{$85.18\pm0.11$}&{$78.25\pm0.19$}&{$69.83\pm0.11$}&{$50.60\pm0.12$}&{$43.74\pm0.23$}&{$86.03\pm1.76$}&{$66.25\pm1.49$}&{$97.27\pm0.21$}\\
	\midrule
	{PUAT}&\bm{$92.27\pm1.04$}&\bm{$87.79\pm1.63$}&\bm{$80.31\pm1.56$}&\bm{$59.19\pm0.12$}&\bm{$53.37\pm0.00$}&\bm{$88.22\pm0.54$}&\bm{$68.03\pm0.48$}&\bm{$98.75\pm0.08$}\\
	\bottomrule
	\end{tabular}
\end{table*}

\begin{table*}[t]
	\renewcommand{\arraystretch}{}
	\centering
	\caption{
	Performance on \textbf{CIFAR10}.
	}
	\label{Tb:white_cifar10}
	\setlength{\tabcolsep}{0.15cm}
	\begin{tabular}{l|c|c|c|c|c|c|c|c}
	\toprule
	\multirow{3}{*}{Methods}&\multirow{3}{*}{\makecell[c]{Natural \\Accuracy}}&\multicolumn{4}{c|}{Robust Accuracy to RAE}&\multicolumn{3}{c}{Robust Accuracy to UAE}\\
	\cmidrule{3-9}
	& &{PGD-2/255}&{PGD-4/255}&{PGD-8/255}&{AA}&{GPGD-0.01}&{GPGD-0.1}&{USong}\\
	\midrule
	{Regular}&{$80.21\pm0.50$}&{$2.41\pm0.43$}&{$0.02\pm0.0$}&{$0.0\pm0.0$}&{$0.0\pm0.0$}&{$65.87\pm1.54$}&{$34.99\pm2.14$}&{$50.71\pm2.16$}\\
	{TRADES}&{$55.27\pm0.70$}&{$48.65\pm0.49$}&{$42.27\pm0.44$}&{$29.97\pm0.27$}&{$26.77\pm0.64$}&{$63.36\pm1.87$}&{$42.70\pm1.24$}&{$64.64\pm0.91$}\\
	{DMAT}&{$56.57\pm0.90$}&{$49.21\pm0.28$}&{$42.20\pm0.26$}&{$28.69\pm0.25$}&{$23.32\pm0.29$}&{$74.28\pm1.86$}&{$54.16\pm2.10$}&{$81.40\pm0.20$}\\
	{RST}&{$71.80\pm0.58$}&{$65.82\pm0.55$}&{$58.93\pm0.37$}&{$44.73\pm0.17$}&{$41.81\pm1.94$}&{$83.63\pm0.48$}&\bm{$64.99\pm0.43$}&{$89.01\pm0.64$}\\
	\midrule
	{PUAT}&\bm{$83.02\pm0.36$}&\bm{$76.09\pm0.26$}&\bm{$66.76\pm0.21$}&\bm{$45.10\pm0.24$}&\bm{$43.74\pm0.12$}&\bm{$86.78\pm0.29$}&{$64.91\pm0.64$}&\bm{$92.55\pm0.31$}\\
	\bottomrule
	\end{tabular}
\end{table*}

\begin{table*}[!t]
	\renewcommand{\arraystretch}{}
	\centering
	\caption{
	Performance on \textbf{CIFAR100}.
	}
	\label{Tb:white_cifar100}
	\setlength{\tabcolsep}{0.15cm}
	\begin{tabular}{l|c|c|c|c|c|c|c|c}
	\toprule
	\multirow{3}{*}{Methods}&\multirow{3}{*}{\makecell[c]{Natural \\Accuracy}}&\multicolumn{4}{c|}{Robust Accuracy to RAE}&\multicolumn{3}{c}{Robust Accuracy to UAE}\\
	\cmidrule{3-9}
	& &{PGD-2/255}&{PGD-4/255}&{PGD-8/255}&{AA}&{GPGD-0.01}&{GPGD-0.1}&{USong}\\
	\midrule
	{Regular}&\bm{$59.02\pm0.31$}&{$12.14\pm0.44$}&{$1.04\pm0.02$}&{$0.02\pm0.0$}&{$0.0\pm0.0$}&{$25.28\pm0.29$}&{$10.59\pm0.31$}&{$26.59\pm0.62$}\\
	{TRADES}&{$34.23\pm0.15$}&{$28.10\pm0.12$}&{$22.46\pm0.14$}&{$14.40\pm0.01$}&{$11.99\pm0.18$}&{$20.45\pm0.19$}&{$8.43\pm0.30$}&{$24.56\pm0.27$}\\
	{DMAT}&{$30.81\pm0.77$}&{$22.74\pm0.43$}&{$15.20\pm0.84$}&{$5.88\pm0.75$}&{$2.68\pm0.73$}&{$21.92\pm1.50$}&{$10.88\pm0.54$}&{$31.81\pm2.19$}\\
	{RST}&{$50.34\pm0.52$}&{$40.30\pm0.52$}&{$31.32\pm0.50$}&{$17.36\pm0.35$}&{$13.36\pm0.48$}&{$32.00\pm0.15$}&\bm{$16.26\pm0.18$}&\bm{$41.82\pm0.21$}\\
	\midrule
	{PUAT}&{$51.95\pm0.45$}&\bm{$42.70\pm0.43$}&\bm{$33.54\pm0.01$}&\bm{$18.45\pm0.38$}&\bm{$17.16\pm0.46$}&\bm{$33.00\pm0.10$}&{$16.21\pm0.16$}&{$38.32\pm0.43$}\\
	\bottomrule
	\end{tabular}
\end{table*}

\subsection{Performance (RQ1)}
\label{Sec:performance}
Tables \ref{Tb:white_tiny}, \ref{Tb:white_imagenet32}, \ref{Tb:white_svhn}, \ref{Tb:white_cifar10}, and \ref{Tb:white_cifar100} show the performances of PUAT and the baseline methods on all the datasets, from which we can make the following observations:

\begin{itemize}
\item On all datasets, PUAT surpasses the baseline AT methods in robust accuracy under all attacks regardless of the attacking strength ($\epsilon$), with the exception of GPGD-0.1 on CIFAR10 and CIFAR100, and the exception of USong-0.01 on CIFAR100. At the same time, PUAT achieves the highest natural accuracy compared with the baseline AT methods. Therefore, PUAT is the only AT approach that simultaneously increases standard generalization and adversarial robustness against all attacks on all datasets, due to its ability to align the AE distribution and natural data distribution, which results in a consistent generalization on AEs and natural samples. This result demonstrates PUAT’s capacity to improve the tradeoff between the standard generalization and the adversarial robustness.

\item PUAT beats all all defensive methods against RAE attacks in all cases and UAE attacks in most cases. This result shows how effective PUAT is at providing comprehensive adversarial robustness against both RAE and UAE attacks. We contend that the comprehensive adversarial robustness is made possible by our innovative notion to generalize the UAE concept to imperceptibly perturbed natural examples, whether observable or unobservable, which makes PUAT able to cover RAEs in a unified way.

\item The robustness of all methods decreases as the attack strength ($\epsilon$) grows up. However, under all the attack strengths, PUAT consistently outperforms the baseline methods in the robustness against RAEs, and the robustness against UAEs in most cases. This suggests that AT on UAEs provides better adversarial robustness because the adversary of UAEs can be improved by utilizing semi-supervised data, as told by Theorem \ref{Theorem:UAE_adversary_finite}.

\item We also note that PUAT outperforms DMAT which is an AT method also based on generative model. This is because DMAT only uses labeled data, while PUAT can benefit from both labeled and unlabeled data. In fact, in most cases on CIFAR100 where labeled data per each class is extemely sparse,  DMAT performs poorly than the other baselines, which shows DMAT is unsuitable for sparse labeled data.

\item We observe that PUAT is second to RST in robust accuracy under attacks USong-0.01 and GPGD-0.1 on CIFAR100. We argue that this is partly due to the fact that it is hard to train a triple-GAN on a dataset with too many classes like CIFAR100 which contains 100 classes. One possible solution is to improve the model complexity of the G-C-D GAN, which is worth exploring in future. 

\end{itemize}

\subsection{
Improvement of The Tradeoff (RQ2)
}
\label{sec:tradeoff}
\begin{figure*}[t]
\centering
		\subfigure[
		PGD-8/255
		]{
		\includegraphics[scale=0.29]{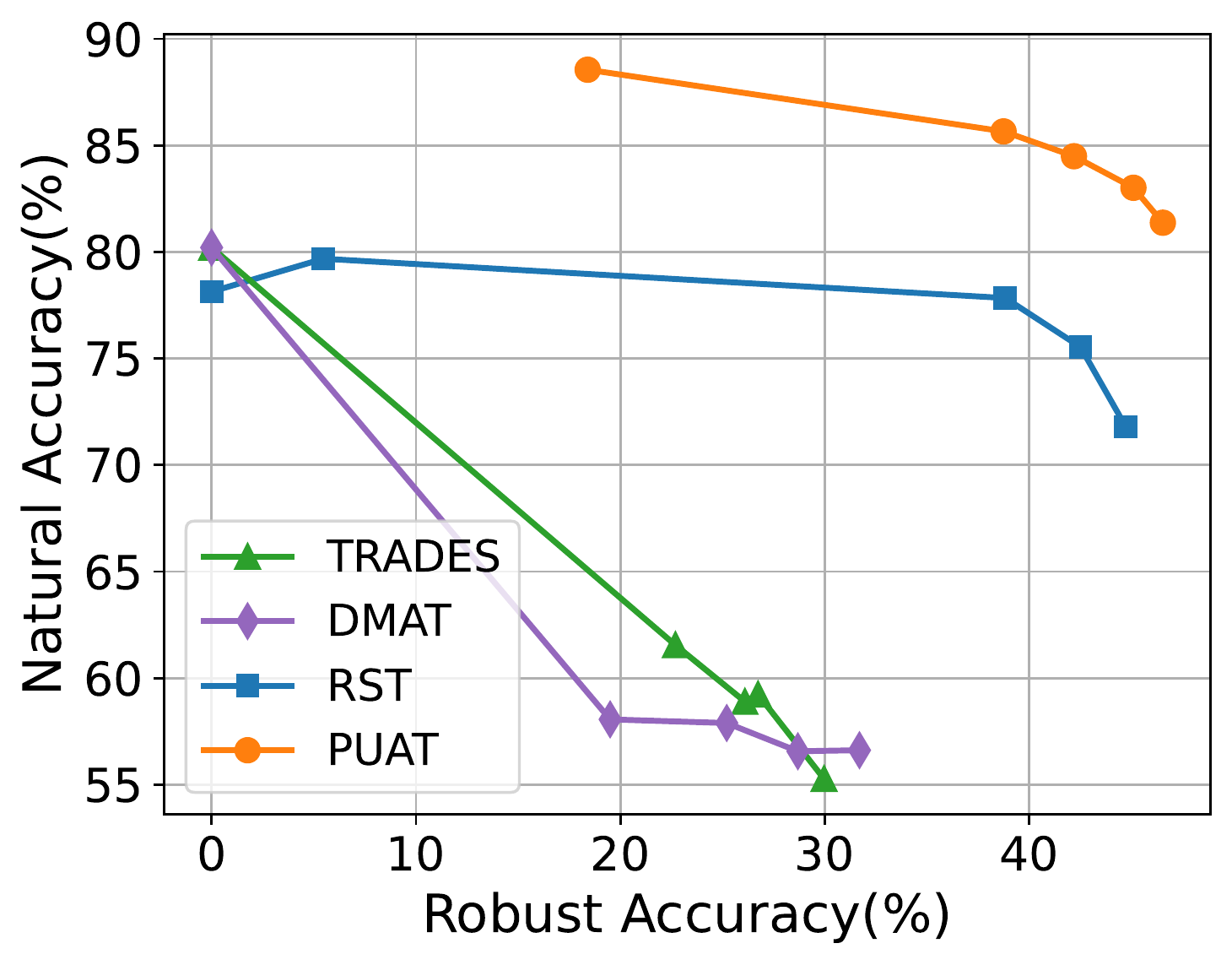} 
		\label{fig:hyper_pgd}}
		\subfigure[
		AA
		]{
		\includegraphics[scale=0.29]{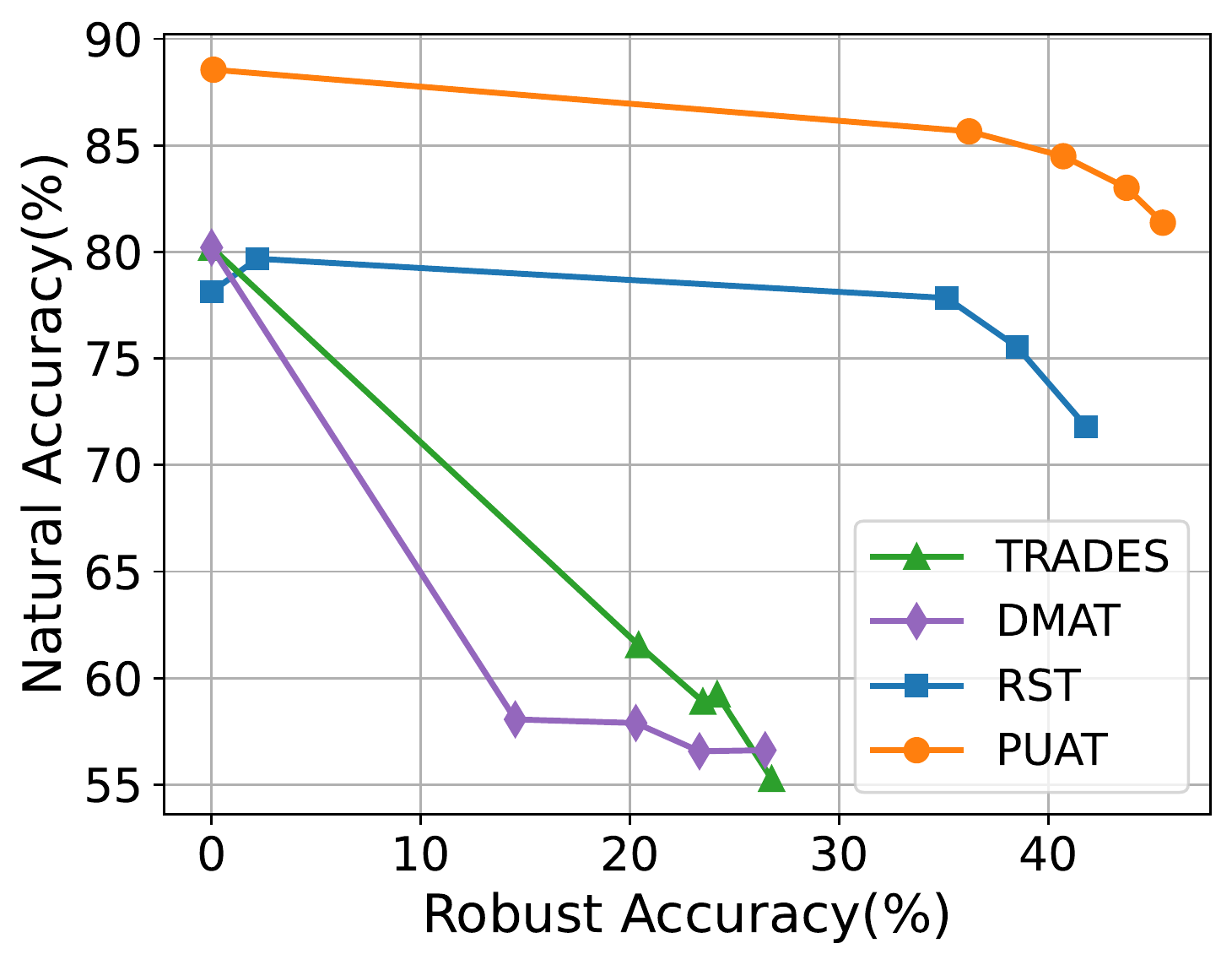}
		 \label{fig:hyper_aa}}
		\subfigure[
		GPGD-0.1
		]{
		\includegraphics[scale=0.29]{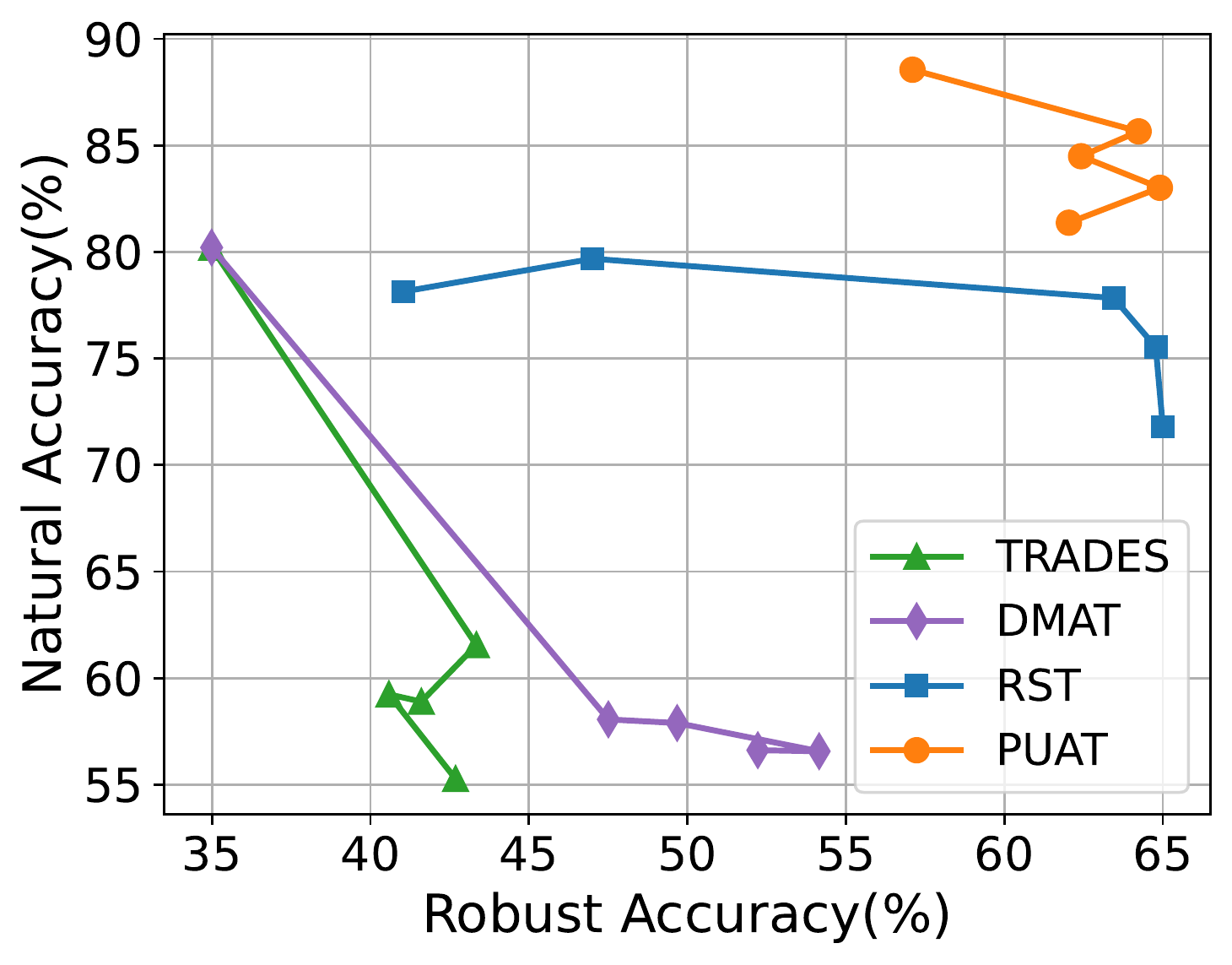} 
		\label{fig:hyper_gpgd}}
		\subfigure[
		USong
		]{
		\includegraphics[scale=0.29]{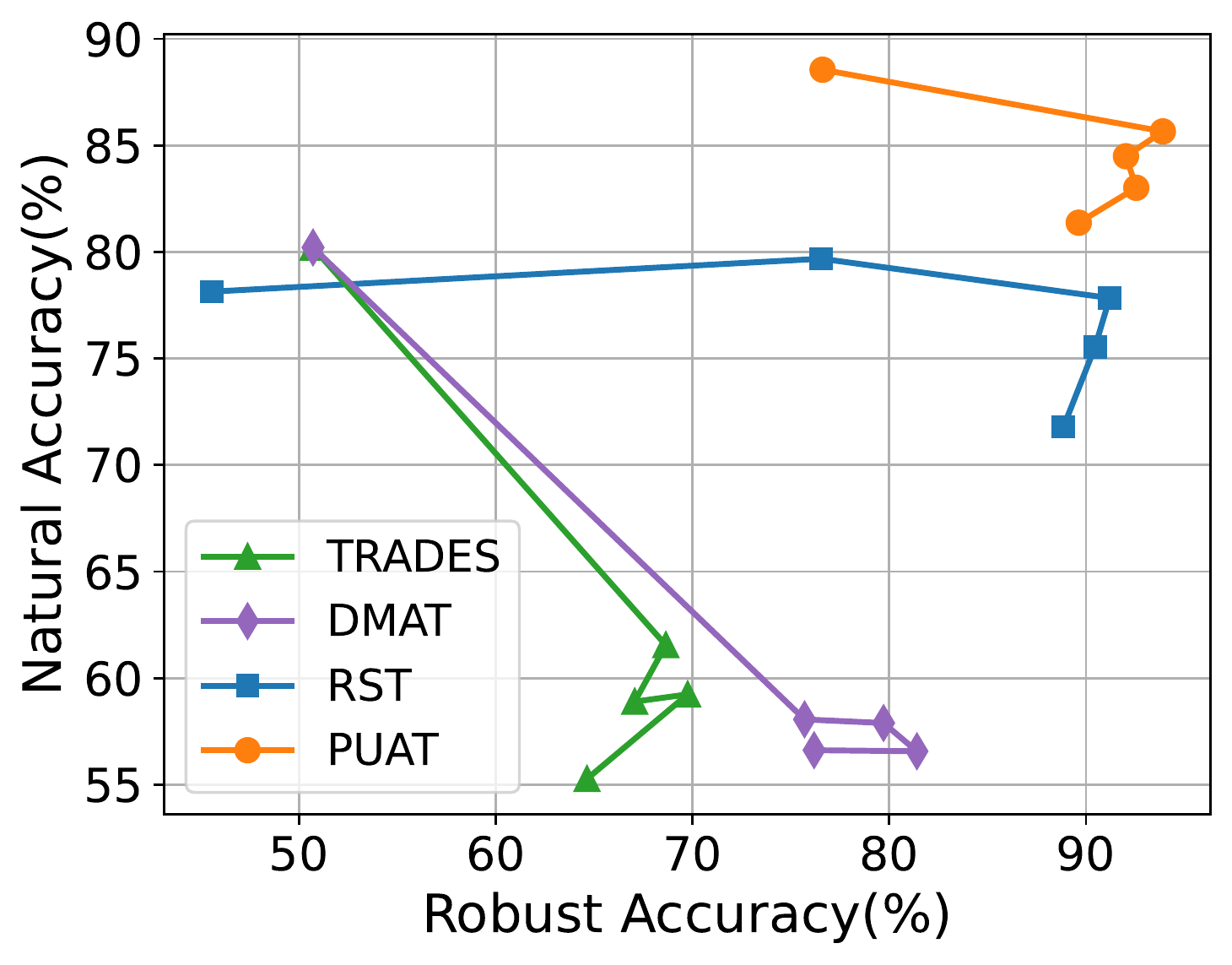} 
		\label{fig:hyper_usong}}
\caption{
Pareto front curves of natural accuracy vs. robust accuracy on CIFAR10 under different adversarial attacks including (a) PGD-8/255, (b) AA, (c) GPGD-0.1, and (d) USong. Each curve consists of points generated by setting $\beta$ with various values. 
}
\label{fig:RQ2_cifar10}
\end{figure*}
Now we examine if PUAT can improve the tradeoff between standard generalization and adversarial robustness. For this purpose, in Figure \ref{fig:RQ2_cifar10} we show the Pareto front curves of natural accuracy vs. robust accuracy under adversarial attacks PGD-8/255, AA, GPGD-0.1 and USong. We can see that under all adversarial attacks, the points of PUAT's curve always lies to the upper right of the curves of the baseline methods, which implies that PUAT can achieve a better natural accuracy as well as a better robust accuracy than the baseline methods, or in other words, PUAT can improve the tradeoff between standard generalization and adversarial robustness. This is because (1) By using G-C-D GAN to align the generator distribution $P_G(x, y)$, the target classifier distribution $P_C(x, y)$ and the natural data distribution $P(x, y)$, PUAT is able to reconcile the robust generalization on AEs and the standard generalization on natural data, which is asserted by Theorem \ref{Theorem:GCD_finite}; and (2) PUAT conducts AT on UAEs which results in greater robustness since UAE adversary will be improved by both labeled and unlabeled data, which is asserted by Theorem \ref{Theorem:UAE_adversary_finite}. Additionally, we also observe that RST achieves a better tradeoff than TRADES and DMAT, which is partly because RST also utilizes unlabeled samples to augment the data used for AT.  

\subsection{Ablation Study (RQ3)}
\begin{figure}[t]
\centering
\subfigure[
effect of UAE
]{
\includegraphics[scale=0.273]{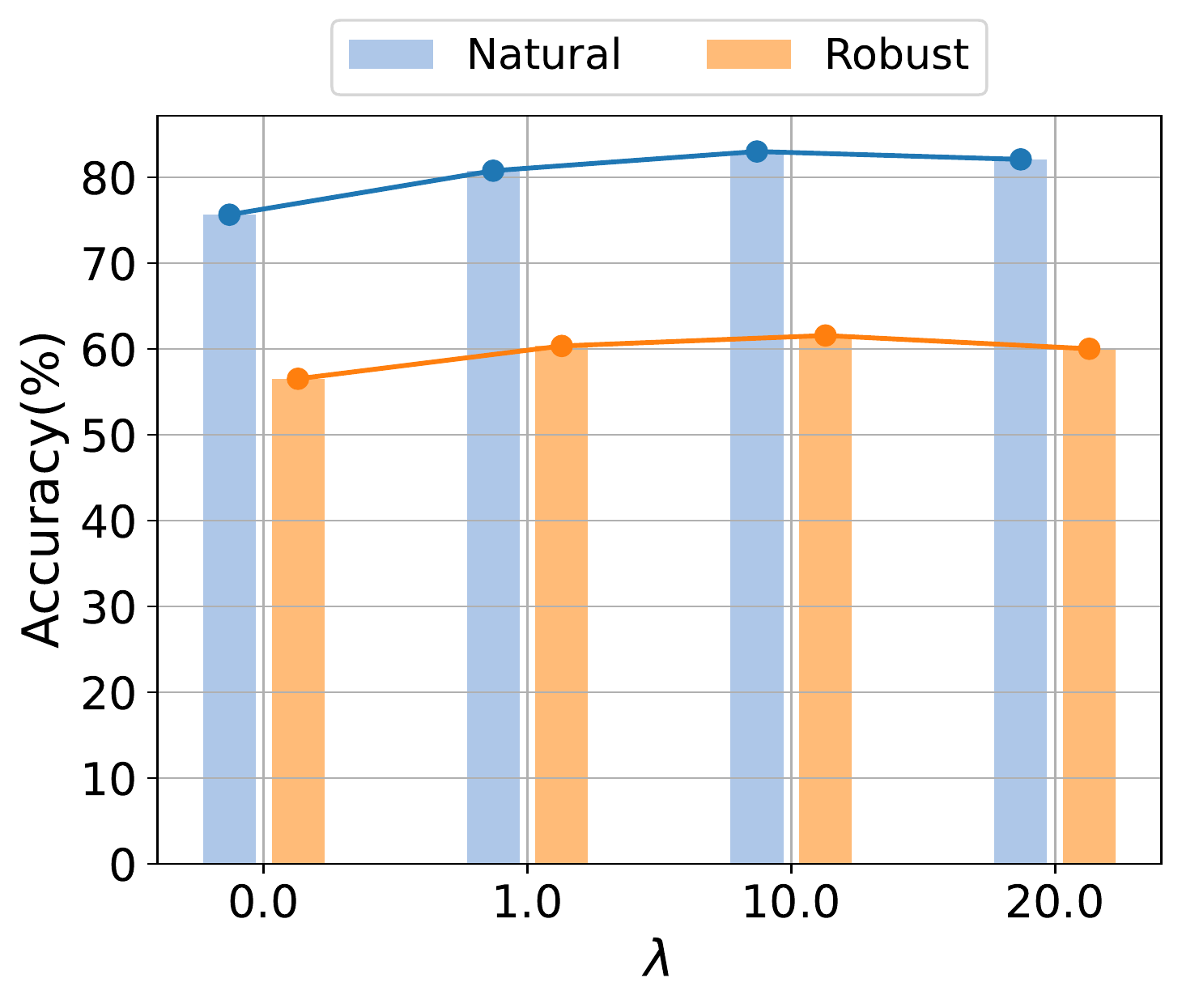} 
\label{fig:uae}}
\subfigure[
effect of (un)labeled data
]{
\includegraphics[scale=0.25]{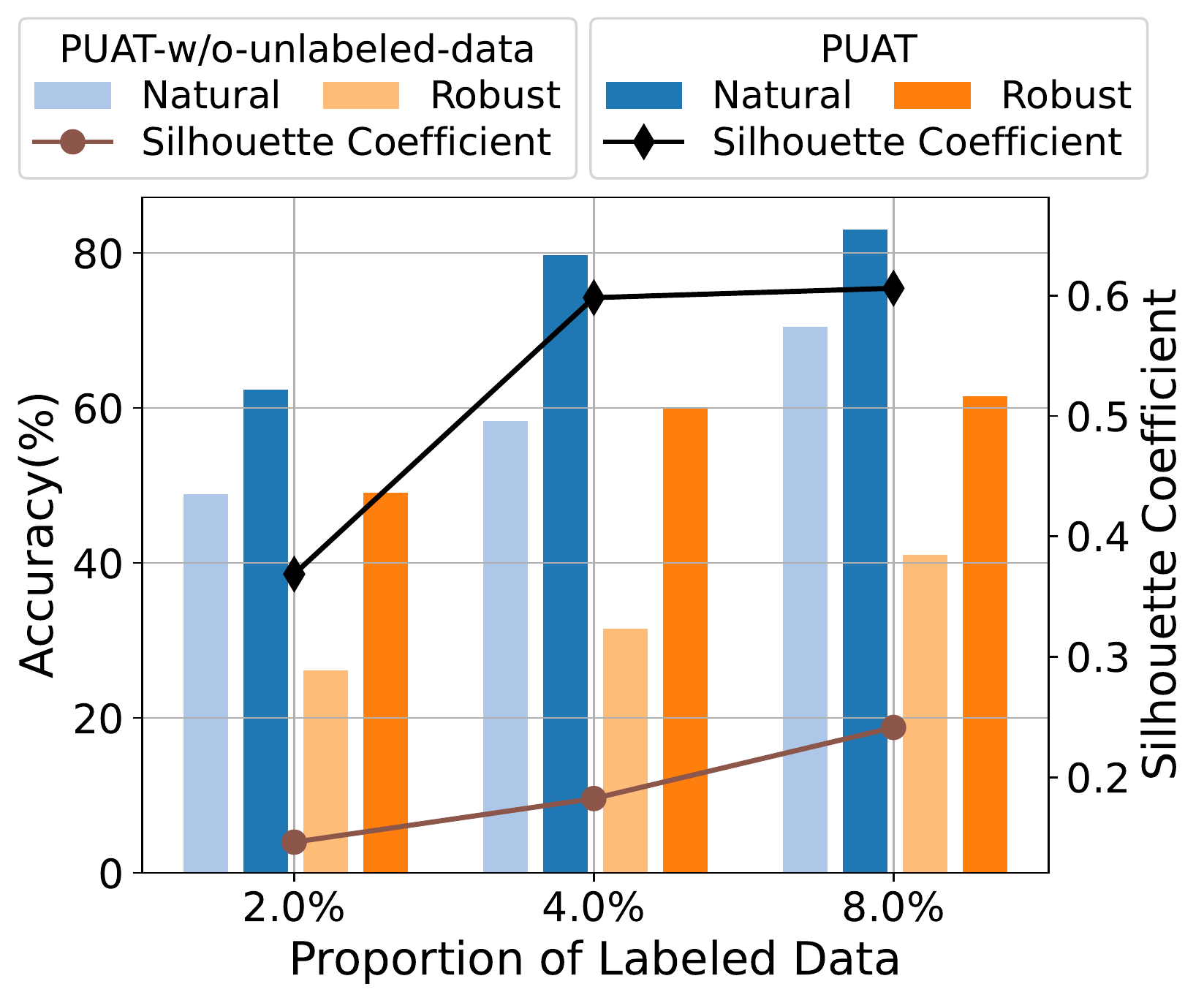} 
\label{fig:unlabel}}
\caption{
Effect of UAE and (un)labeled data. Subfigure (a) shows the natural accuracy and robust accuracy of PUAT under different $\lambda$ on CIFAR10. Subfigure (b) shows the natural accuracy and robust accuracy of PUAT and PUAT-w/o-unlabeled-data, and the Silhouette Coefficients, over different amount of labeled data on CIFAR10, where PUAT-w/o-unlabeled-data is a variant of PUAT without using unlabeled data during the training. The robust accuracy in both (a) and (b) is the average of the robust accuracies on PGD-8/255, AA, GPGD-0.1 and USong, which reflects the comprehensive adversarial robustness. 
}
\label{fig:ablation_cifar10}
\end{figure}

\begin{figure}[t]
\centering
\subfigure[
Robust Accuracy
]{
\includegraphics[scale=0.265]{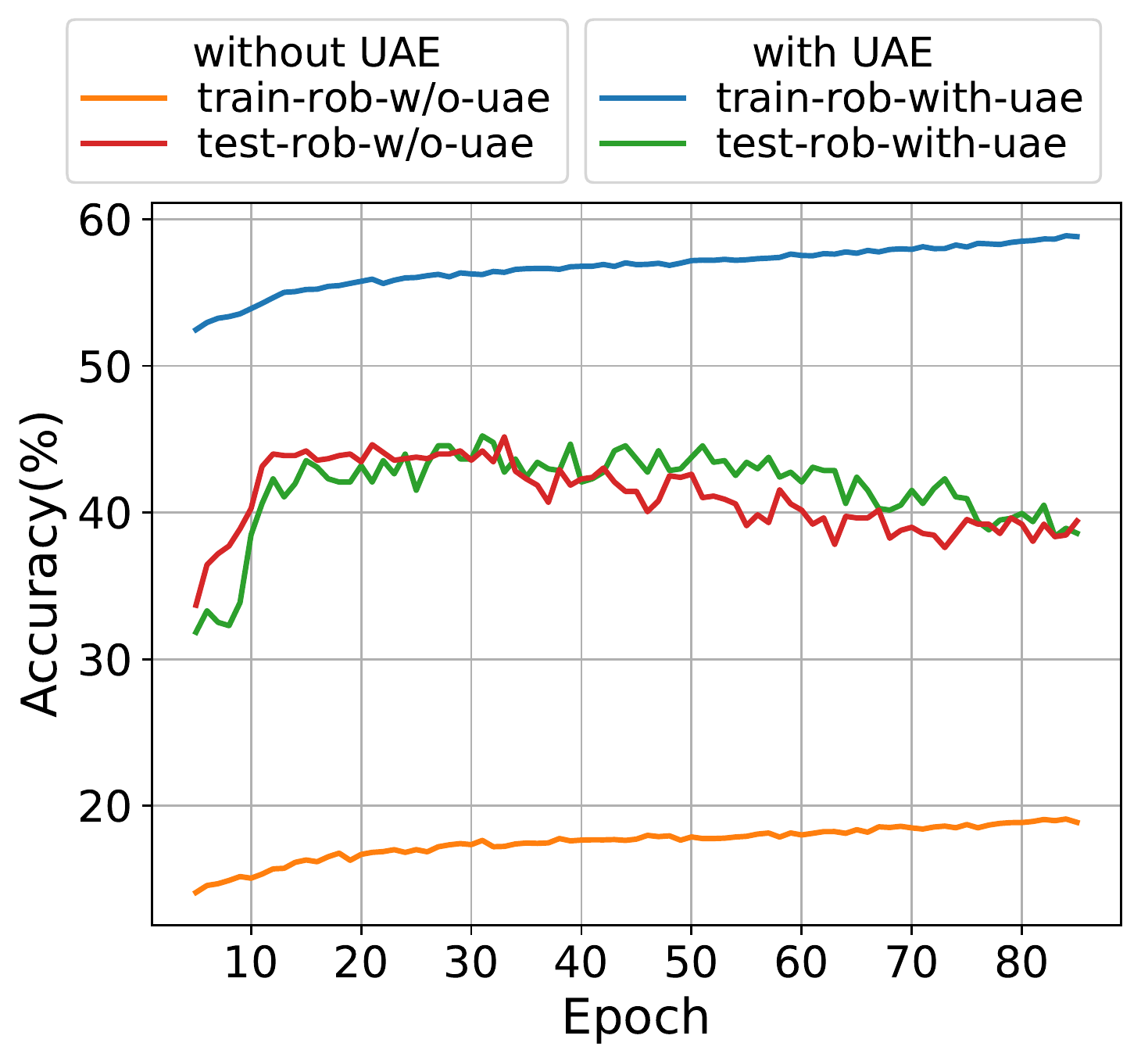}
 \label{fig:learning_rob}}
\subfigure[
Natural Accuracy
]{
\includegraphics[scale=0.265]{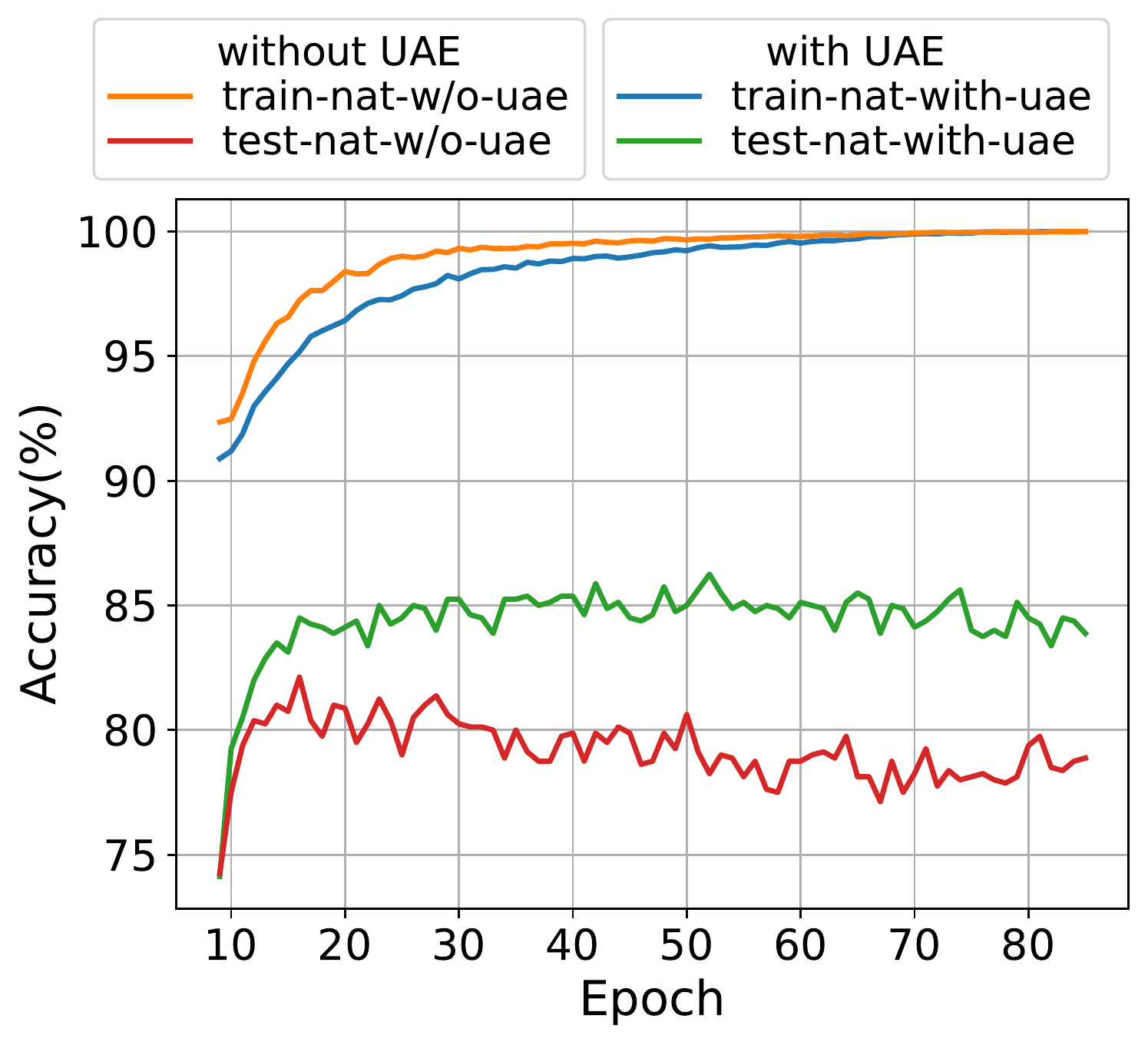} 
\label{fig:learning_acc}}
\caption{
Learning curves of PUAT with UAEs ($\lambda=10.0$) and without UAEs ($\lambda=0.0$). In both Subfigures (a) and (b), blue and green curves are the robust/natural accuracy curve of training with UAEs and its corresponding testing robust/natural accuracy curve, respectively, while orange and red curves are the robust/natural accuracy curve of training withouth UAEs and its corresponding testing robust/natural accuracy curve, respectively.
}
\label{fig:learning_curve}
\end{figure}

\begin{figure*}[t]
\centering
\subfigure[
Alignment on CIFAR10
]{
\includegraphics[scale=0.18]{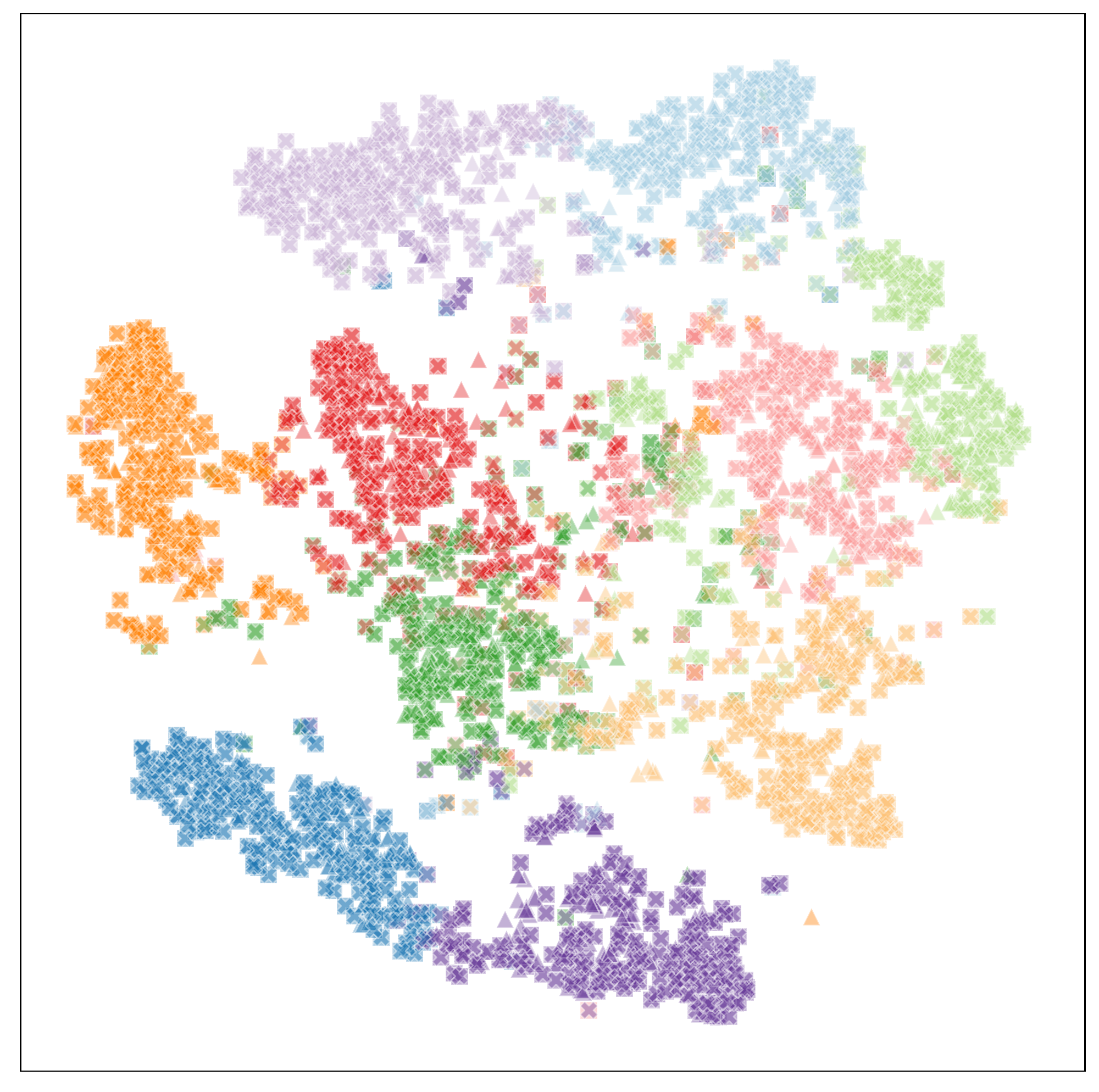} 
\label{fig:align_1_cifar10}}
\subfigure[
Misalignment on CIFAR10
]{
\includegraphics[scale=0.18]{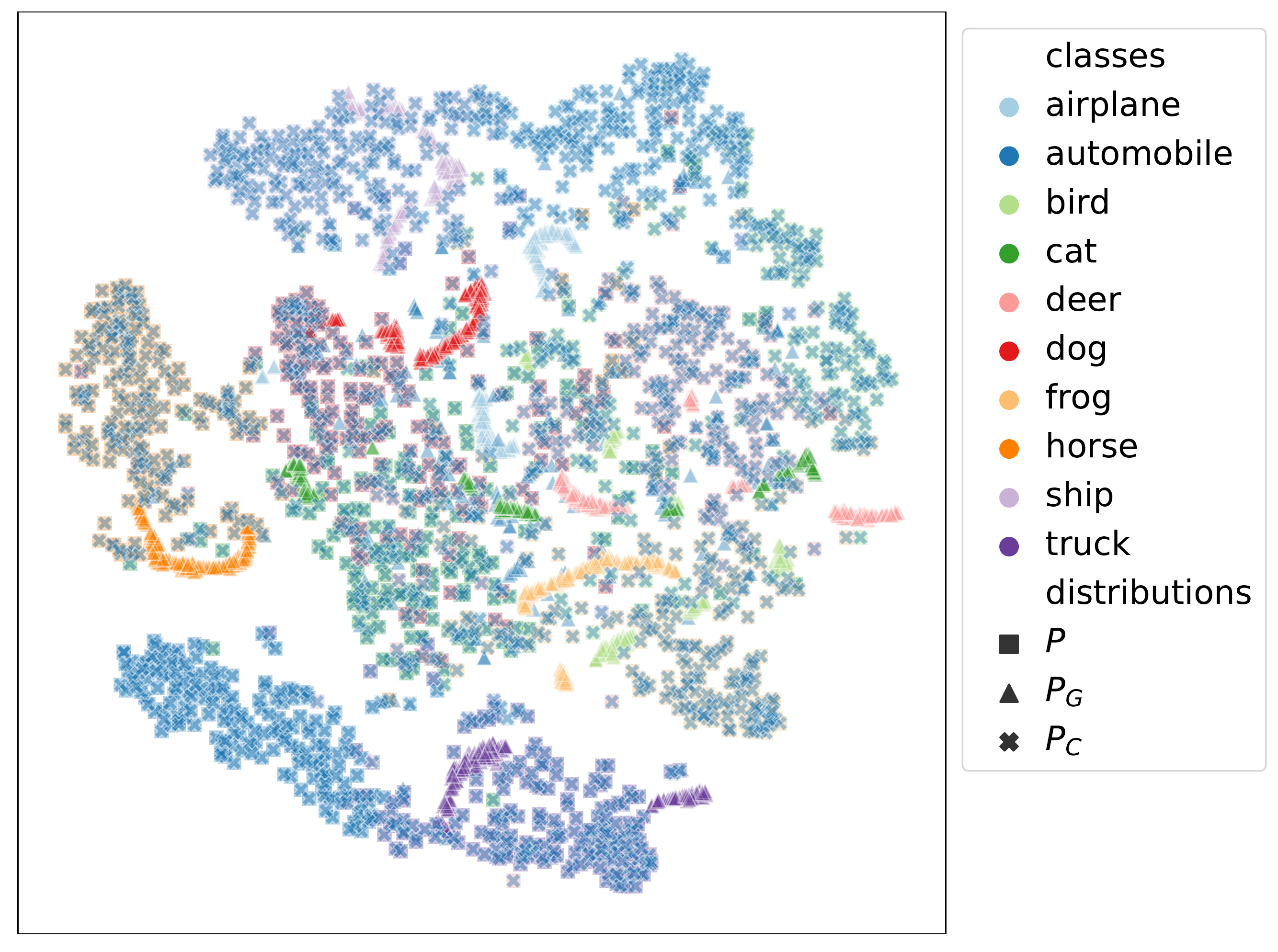} 
\label{fig:align_0_cifar10}}
\subfigure[
Alignment on SVHN
]{
\includegraphics[scale=0.18]{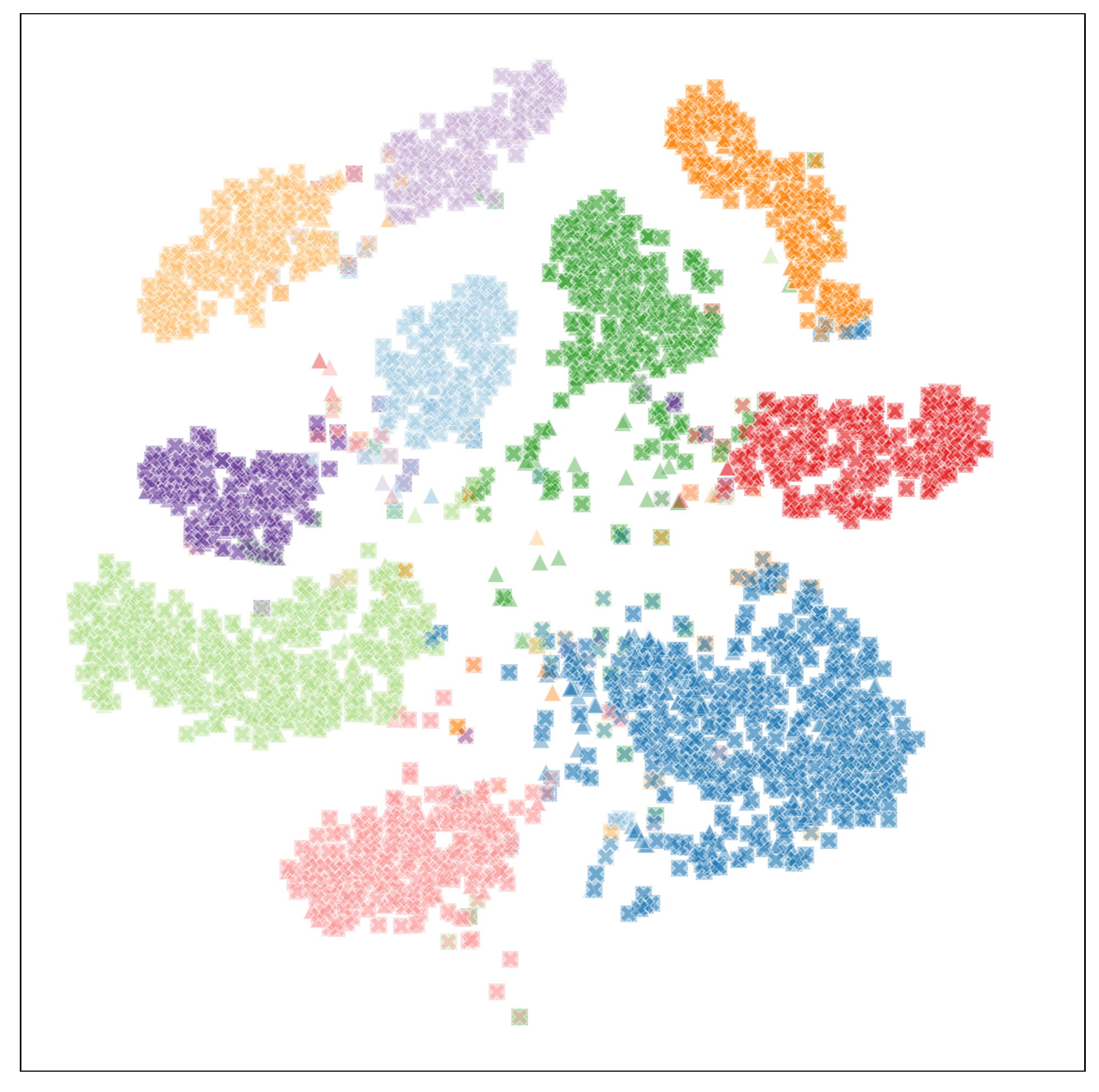} 
\label{fig:align_1_svhn}}
\subfigure[
Misalignment on SVHN
]{
\includegraphics[scale=0.18]{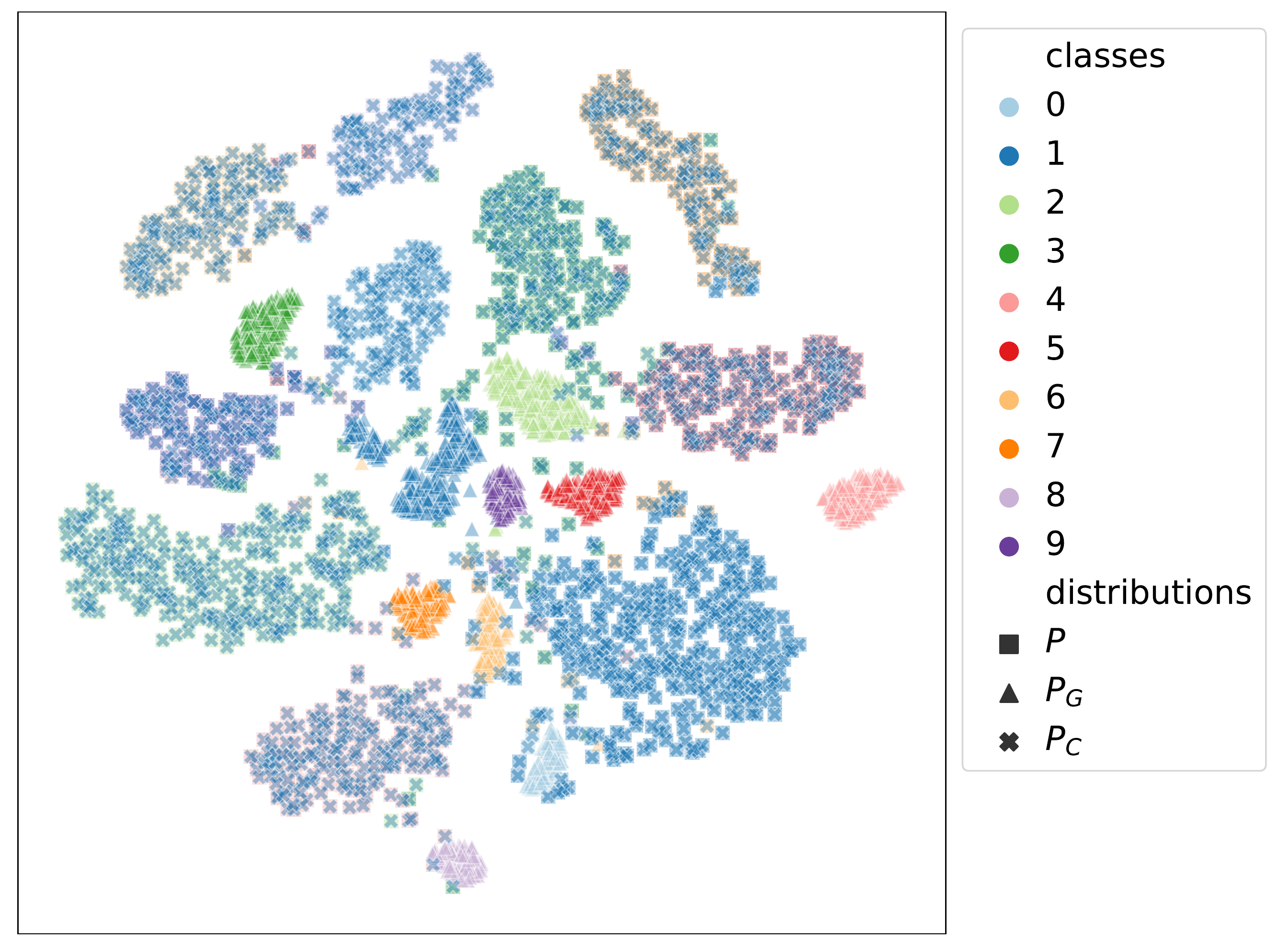}
\label{fig:align_0_svhn}}
\caption{
Visualized distribution alignment or misalignment, where different colors denote classes and different patterns denote distributions.
}
\label{fig:alignment}
\end{figure*}

Now we investigate how UAEs and (un)labeled data influence the performance of PUAT.

\subsubsection{Effect of UAE}
\label{Sec:UAE}

To investigate the UAE's effect on PUAT, we first train PUAT with different $\lambda \in \{0.0, 1.0, 10.0, 20.0 \}$ in Equation (\ref{Eq:loss_new}), where $\lambda=0.0$ means the training without UAEs, and then check its testing natural accuracy and robust accuracy. The result is shown in Figure \ref{fig:uae}, from which we can see as the importance of UAEs (indicated by $\lambda$) increases, both natural accuracy and robust accuracy increase, which verifies that UAEs can benefit PUAT by boosting standard generalization and adversarial robustness simultaneously. We also note that when $\lambda$ exceeds 10, the performance of PUAT begins to drop, which shows that excessively large $\lambda$ may lead to overfitting of both natural sample distribution and UAE distribution.

To gain a deeper understanding of the impact of UAE on the learning of PUAT, we also plot the learning curves of PUAT both with and without UAEs in Figure \ref{fig:learning_curve}. From Figure \ref{fig:learning_rob} we can see that regardless of whether UAEs are used or not, the training robust accuracy curves both rise with the number of training epochs, while their corresponding testing robust accuracy curves both first rise and then fall gradually. This results indicate that there exist robust overfitting \cite{rice2020overfitting}, i.e., the overfitting on AEs, whether or not UAEs are used. To address this issue, following the idea proposed in \cite{rice2020overfitting}, we adopt validation-based early stopping to regularize the training of PUAT. On the other hand, Figure \ref{fig:learning_acc} shows that whether UAEs are used or not, the training natural accuracy curves (blue and yellow curves) both first rise and then gradually stabilize. However, when UAEs are not used for training, the testing natural accuracy curve (red curve) first rises and then drops, which indicates overfitting occurs when training without UAEs. In sharp contrast, when training with UAEs, the testing natural accuracy curve (green curve) first rises and then converges gradually, which suggests that using UAEs helps mitigate the overfitting on natural samples.

\subsubsection{Effect of natural samples}
\label{Sec:nat-samples}
Recall that Theorem \ref{Theorem:GCD_finite} states that the more the labeled natural samples are used for training, the better the distribution alignment, and the better the tradeoff between standard generalization and comprehensive adversarial robustness. Now we empirically verify this conclusion by checking the testing natural accuracy, robust accuracy and the Silhouette Coefficient \cite{rousseeuw1987silhouettes} after training with various amount of natural samples. The results are shown in Figure \ref{fig:unlabel}, where we compare PUAT with PUAT-w/o-unlabeled-data (the variant of PUAT without using unlabeled natural samples for training). 
\par
At first, from Figure \ref{fig:unlabel} we can see that regardless of whether or not unlabeled samples are used, the standard generalization (measured by the natural accuracy), the comprehensive adversarial robustness (measured by the average robust accuracy), and the quality of the clustering of the samples from different distributions ($P$, $P_G$, $P_C$) of each class (measured by Silhouette Coefficient), are all improved as the amount of labeled samples used for training increases. These results validate that both PUAT and PUAT-w/o-unlabeled-data can achieve better distribution alignment and better tradeoff with more labeled samples for training due to G-C-D GAN's ability to align the UAE distribution with natural sample distribution, which is consistent with the conclusions of Theorem \ref{Theorem:GCD_finite}. 
\par
At the same time, it is noteworthy that PUAT performs significantly better than PUAT-w/o-unlabeled-data even with the labeled data of only 2\%, which shows that PUAT can improve the performance by leveraging unlabeled data even though the labeled samples are extremely sparse. This is because unlabeled data helps G-C-D GAN more accurately capture real distribution $P(x,y)$, and consequently, enables the target classifier $C$ to generalize consistently over the aligned distributions of natural samples and AEs.

\begin{table*}[h]
	\renewcommand\arraystretch{}
	\centering
	\caption{The sample number of different classes from different distributions. }
	\label{Tb:numbers}
	\setlength{\tabcolsep}{0.04cm}
	\begin{tabular}{l|c|c|c|c|c|c|c|c|c|c|c|c|c|c|c|c|c|c|c|c}
	\toprule
	\multirow{2.6}{*}{Distribution}&\multicolumn{10}{c|}{CIFAR10}
	&\multicolumn{10}{c}{SVHN}\\
	\cmidrule{2-21}
	&{"airplane"}&{"automobile"}&{"bird"}&{"cat"}&{"deer"}&{"dog"}&{"frog"}&{"horse"}&{"ship"}&{"truck"}
	&{"0"}&{"1"}&{"2"}&{"3"}&{"4"}&{"5"}&{"6"}&{"7"}&{"8"}&{"9"}\\
	\midrule
	{$P(y)$}&{200}&{200}&{200}&{200}&{200}&{200}&{200}&{200}&{200}&{200}
	&{134}&{392}&{319}&{221}&{194}&{183}&{152}&{155}&{128}&{122}\\
	{$P_G(y)$}&{200}&{200}&{200}&{200}&{200}&{200}&{200}&{200}&{200}&{200}
	&{134}&{392}&{319}&{221}&{194}&{183}&{152}&{155}&{128}&{122}\\
	{$P_C(y)$}&{208}&{201}&{180}&{165}&{201}&{210}&{227}&{187}&{217}&{204}
	&{143}&{421}&{306}&{199}&{203}&{183}&{152}&{152}&{111}&{131}\\
	{misaligned $P_C(y)$}&{0}&{2000}&{0}&{0}&{0}&{0}&{0}&{0}&{0}&{0}
	&{0}&{2000}&{0}&{0}&{0}&{0}&{0}&{0}&{0}&{0}\\
	\bottomrule
	\end{tabular}
\end{table*}

\subsection{Visual Study (RQ4)}
\label{Sec:visual}
Now, we visually demonstrate PUAT's capability to align the natural data distribution $P(x, y)$, the distribution $P_G(x, y)$ learned by the generator, and the distribution $P_C(x, y)$ learned by target classifier, using the testing sets of CIFAR10 and SVHN which both contain 10 classes. In particular, for each class $y_i$ ($1 \le i \le 10$) of each testing set, we invoke $G$ to generate a set of $\{ (x_g, y_i)\}$ such that $\big| \{ (x_g, y_i) \sim P_G(x, y)\} \big| = \big| \{ (x_l, y_i) \sim P(x, y)\} \big|$. At the same time, all the testing examples are also fed into $C$, and those classified as $y_i$ make up $\{ (x_c, y_i) \sim P_C \}$. At last, we plot the sample points $\{ (x, y_i)\}$, $\{ (x_g, y_i)\}$, and $\{ (x_c, y_i)\}$ in Figure \ref{fig:alignment} by using t-SNE algorithm, where different colors denote classes and different patterns denote distributions.
\par
Figures \ref{fig:align_1_cifar10} and \ref{fig:align_1_svhn} show the results of the PUAT on CIFAR10 and SVHN, respectively. From Figures \ref{fig:align_1_cifar10} and \ref{fig:align_1_svhn} we can see that the data points are distributed in different colored clusters with clear boundaries, where the examples from different distributions (with different patterns) but of the same class (with the same color) are grouped in the same cluster, and there are essentially no data points in any cluster that have a different color from the majority. However, once we remove the extended AT from PUAT so that Theorem \ref{Theorem:GCD} does not hold, the cluster boundaries become blurred and the colors in the same cluster become impure, as shown in Figures \ref{fig:align_0_cifar10} and \ref{fig:align_0_svhn}.
\par
Besides, Figures \ref{fig:align_1_cifar10} and \ref{fig:align_1_svhn} also show that when the distributions are aligned, the samples of every class from $P_G$ distributed evenly in the cluster. On the contrary, Figures \ref{fig:align_0_cifar10} and \ref{fig:align_0_svhn} show that when the extended AT is removed from PUAT, the samples in clusters from $P_G$ concentrate to a small area, which indicates the mode collapse of G-C-D GAN. These results confirm that the distribution alignment brought by the extended AT of PUAT can avoid mode collapse of G-C-D GAN.
\par
Therefore, we can contend that PUAT aligns the three conditional distributions $P$, $P_C$ and $P_G$. Furthermore, as shown in Table \ref{Tb:numbers}, when the distributions are aligned, the number of examples belonging to $y_i$ is approximately the same in different distributions, which indicates the marginal distributions $P(y)= P_{G}(y) = P_{C}(y)$, so we can confirm that $P(x, y)$, $P_C(x, y)$ and $P_G(x, y)$ are aligned by PUAT. 

\subsection{
Training Time
}
\label{sec:time}
\begin{figure}[t]
\centering
\includegraphics[scale=0.35]{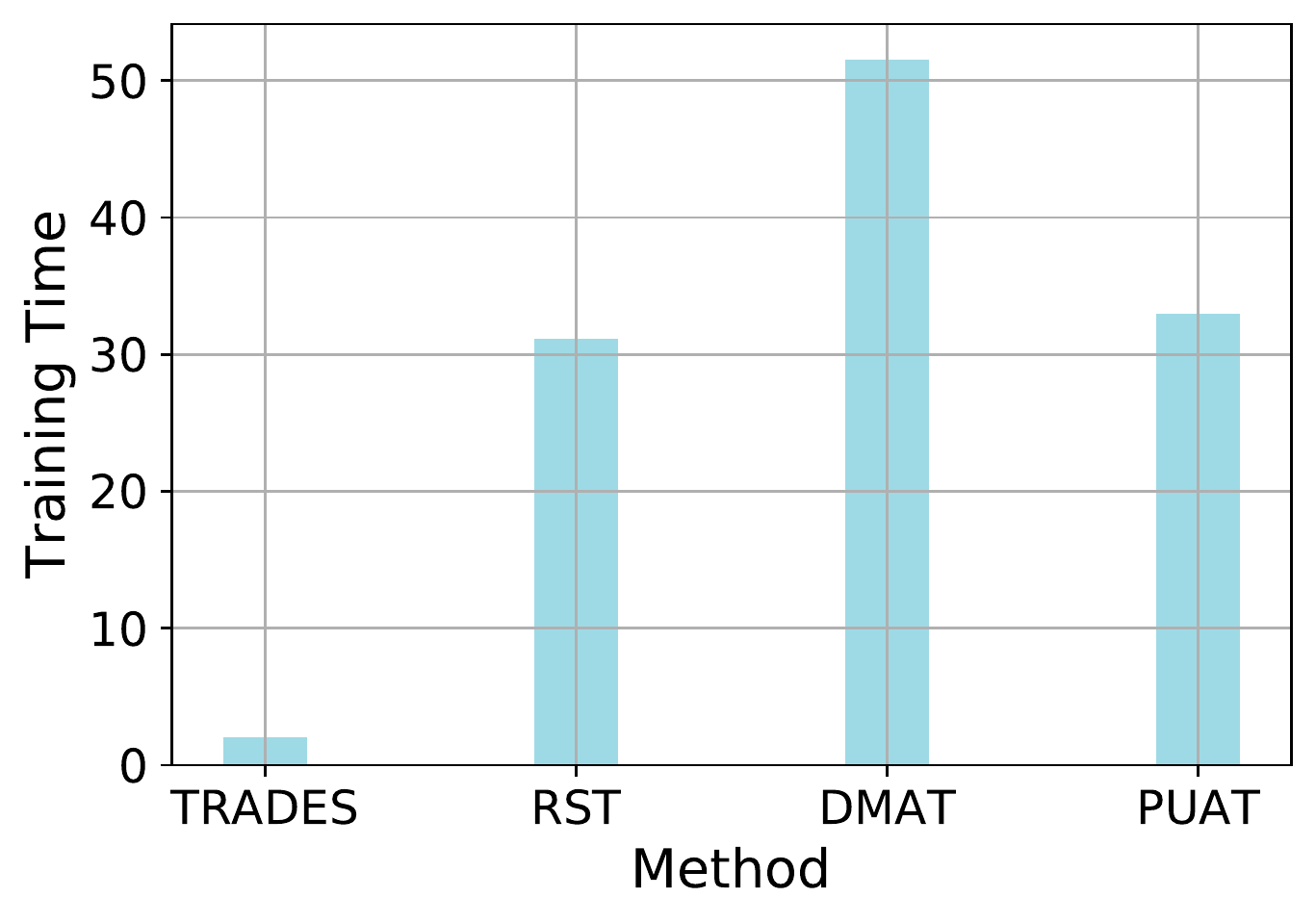}
\caption{	
Training time comparison between AT baselines and PUAT. The training time of each method is measured by its ratio relative to the time of regular training. Only the time taken until early stopping is factored into the calculations. The average values across Tiny ImageNet, SVHN, CIFAR10, and CIFAR100 datasets are reported.
}
\label{fig:time}
\end{figure}
The training time of each method is measured by its ratio relative to the time of regular training. As we can see from Figure \ref{fig:time}, the training time of PUAT is significantly shorter than that of DMAT and slightly longer than that of RST. This is attributable to the rapid convergence exhibited by PUAT, as consistently demonstrated by its learning curves in Figure \ref{fig:learning_curve}. We also see that the training of TRADES takes the shortest time, which is because it only uses labeled data. In summary, Figure \ref{fig:time} together with Figure \ref{fig:RQ2_cifar10} shows that PUAT can achieve a superior balance between standard generalization and adversarial robustness at an acceptable additional time cost.
\section{Related Work}
In this section, we briefly introduce the related works from three domains, adversarial attacks, adversarial training, and standard generalizability, and comment their differences from our work. 

\subsection{Adversarial Attacks}

\label{sec:attacks}

Generally, adversarial attacks aim to fool a target model well-trained on natural data with AEs that are imperceptible to humans \cite{miller2020adversarial,zhang2021adversarial}. The existence of AEs is first discovered by the seminal works of Szegedy \textit{et al}. \cite{szegedy2013intriguing} and Goodfellow \textit{et al}. \cite{goodfellow2014explaining}. In \cite{szegedy2013intriguing}, Szegedy \textit{et al}. propose an attack method called L-BFGS which can imperceptibly perturb a natural example with respect to its second-order gradient to cause misclassifications. In \cite{goodfellow2014explaining}, Goodfellow \textit{et al}. further propose the fast gradient sign method (FGSM), which searches adversarial perturbations just based on a single first-order gradient step. Following the idea of \cite{goodfellow2014explaining}, various adversarial attack methods have been proposed by enhancing the first-order method. For example, Madry \textit{et al}. \cite{madry2017towards} propose a projected gradient descent (PGD) method which can be regarded as an iterative FGSM projecting AEs into a ball constrained by $L_{\infty}$ norm. Instead of restricting the perturbation norms to a fixed value, Moosavi-Dezfooli \textit{et al}. \cite{moosavi2016deepfool} propose an optimization based attack method Deepfool to seek adversarial perturbations for differentiable classifiers by minimizing $L_2$ norm. Similarly, Carlini and Wagner \cite{carlini2017towards} propose another optimization based method called C\&W attack, which can generate stronger AEs with perturbations minimizing $L_0$, $L_2$ and $L_{\infty}$ norms at the expense of higher computation complexity. In contrast to gradient based methods, another line of works employs GAN to generate AEs for better imperceptibility. For example, Xiao \textit{et al}. \cite{xiao2018generating} propose AdvGAN to generate perturbations for observed examples, which extends standard GAN with explicit perturbation norm loss to ensure the imperceptibility. Additionally, instead of generating an AE by perturbing a single example, Moosavi-Dezfooli \textit{et al}. \cite{moosavi2017universal} also further propose universal adversarial perturbations which are computed in an example-agnostic fashion and can fool the target classifier on any examples.

The above-mentioned adversarial attacks focus on the AEs which are generated by adding imperceptible perturbations restricted by $L_p$ norm on observed examples, which we call restricted AE (RAE). Recently, some researchers have discovered that there exist unrestricted AEs (UAEs) that are not norm-constrained. To ensure the imperceptibility of UAEs, one line of UAE generation methods manipulates example features by perturbations within perceptual distances that are aligned with human perception \cite{shamsabadi2020colorfool,zhao2020towards,bhattad2020unrestricted,naderi2022generating}. At the same time, some researchers find that in a broader sense, UAEs can be completely new examples built from scratch and dissimilar to any observed example but can still fool classifiers without confusing humans \cite{song2018constructing}. Following this idea, a few of UAE generation methods based on generative models are proposed \cite{song2018constructing,dunn2019adaptive,xiao2022understanding,Tao2022EGM,poursaeed2019fine,wang2019gan}. For example, Song \textit{et al}.\cite{song2018constructing} and Xiao \textit{et al}. \cite{xiao2022understanding} propose to first generate perturbations based on initial random noise and then feed them into a pre-trained conditional GAN (e.g. AC-GAN \cite{odena2017conditional}) to produce UAEs, while Poursaeed \textit{et al}. \cite{poursaeed2019fine} propose to use a Style-GAN \cite{karras2019style} to synthesize UAEs with stylistic perturbations.

The existing works for UAE generation, however, do not explain why UAEs pose threats, or rather, where the imperceptibility of UAEs stems. By contrast, our work proposes a novel viewpoint to view UAEs as perturbed observed and unobserved examples, which provides a logical and provable explanation of UAE's dissimilarity to observed examples, imperceptibility to humans, and the feasibility of comprehensive adversarial robustness. Based on this understanding, we further propose a novel G-C-D GAN for UAE generation, which together with the attacker $A$ can synthesize UAEs that are more adversarial by aligning UAE distribution with natural data distribution.  


\subsection{Adversarial Training}
\label{Sec:relatedAT}

AT has proved to be the strongest principled defensive technique against adversarial attacks \cite{bai2021recent,akhtar2021advances,zhao2022adversarial}. The basic idea of AT is initially proposed by Szegedy \textit{et al}. \cite{szegedy2013intriguing}, which uses a min-max game to incorporate deliberately crafted AEs to the training process of a DNN model to obtain immunity against AEs. Madry \textit{et al}. \cite{madry2017towards} are the first time to theoretically justify the adversarial robustness offered by the min-max optimization used in AT. 

Over the past few years, many enhancements of AT have been proposed to employ various perturbation generation methods to improve adversarial robustness \cite{wang2019bilateral,wang2019improving,cheng2020cat,shafahi2020universal,lee2017generative,stutz2019disentangling}. For example, Wang \textit{et al}. \cite{wang2019bilateral} propose a bilateral adversarial training method that perturbs both normal images and real labels with respect to a smaller gradient magnitude. Wang \textit{et al}. \cite{wang2019improving} propose a misclassification-aware adversarial training (MART) method, which defines a novel adversarial loss with a regularization on misclassified examples. Cheng \textit{et al}. \cite{cheng2020cat} argue that large perturbations may be more adversarial for the robustness of a target model and propose a customized adversarial training (CAT) method which can offer adaptive perturbations for natural examples.  Shafahi \textit{et al}. \cite{shafahi2020universal} propose a universal adversarial training method for the robustness against universal attacks \cite{moosavi2017universal}. Among others, generative model-based AT methods are verified as a promising approach due to the GAN's capability of distributional alignment. For example, to strengthen the adversaries, Lee \textit{et al}. \cite{lee2017generative} propose a generative adversarial trainer (GAT) which employs a GAN to produce adversarial perturbations with respect to the gradients of the natural examples, while Stutz \textit{et al}. \cite{stutz2019disentangling} and Xiao \textit{et al}. \cite{xiao2022understanding} propose to use VAE-GAN to produce AEs whose distribution is aligned with that of the natural examples. As discussed in \ref{sec:attacks}, GAN-based methods strengthen AEs with stronger imperceptibility because of the alignment between the distributions of AEs and the natural examples, which is equivalent to drawing AEs on the manifold (distribution) of the natural examples. 
Another line of works proposed to improve adversarial robustness with data augmentation for AT \cite{wang2023better, rebuffi2021fixing, gowal2021improving, xing2022artificially}. Among them, some studies \cite{wang2023better, rebuffi2021fixing, gowal2021improving} utilize generative model to densify the underlying manifold of natural data, and then conduct an RAE-based AT (e.g. TRADES \cite{zhang2019theoretically}) on the augmented dataset. At the same time, Xing \textit{et al}. \cite{xing2022artificially} theoretically show that the generated data can help improve adversarial robustness. 

The above-mentioned AT methods are basically based on supervised learning, which often suffer from the sparsity of labeled data. Recently, some researchers have discovered that introducing richer unlabeled data can increase the performance of AT, and hence proposed semi-supervised AT methods leveraging partially labeled data \cite{zhang2019defense,carmon2019unlabeled,alayrac2019labels}. For example, Zhang \textit{et al}. \cite{zhang2019defense} argue that the performance of AT will be impaired by the perturbations generated only based on labeled data since they are unable to reflect the underlying distribution of all potential examples, and propose a GAN-based semi-supervised AT method where AEs can preserve the inner structure of unlabeled data. Carmon \textit{et al}. \cite{carmon2019unlabeled} propose a robust self-training (RST) model which first uses an intermediate classifier trained on labeled examples to infer the labels for unlabeled examples, and then generate AEs for AT from the labeled and pseudo-labeled examples. Stanforth \textit{et al}. \cite{alayrac2019labels} propose an unsupervised adversarial training (UAT) method which estimates the worst-case AEs with KL-divergence between the distributions of the perturbed unlabeled examples and the natural unlabeled examples.


Different from the existing AT methods providing adversarial robustness against either RAE or UAE, our PUAT focuses on the concept of generalized UAE which makes it feasible to offer comprehensive adversarial robustness against both UAE and RAE. At the same time, PUAT is both GAN-based and semi-supervised. 


\subsection{Standard Generalizability}
A model's standard generalizability and adversarial robustness are measured by its accuracy on testing natural examples and testing AEs, respectively. Traditional AT methods are based on AEs with norm-constraints, which will lead to manifold-shift between AEs and natural examples \cite{yang2021adversarial,xiao2022understanding}. Therefore, there is a belief that AT forces the target model to generalize on two separated manifolds and hence the tradeoff between standard generalizability and adversarial robustness inevitably comes into being, or in other words, adversarial robustness comes at the expense of standard generalizability \cite{madry2017towards,tsipras2018robustness,zhang2019theoretically,yang2021adversarial}.

However, some researchers have opposing opinions that the tradeoff may be not inevitable, with support of recent empirical evidences and theoretical analyses \cite{stutz2019disentangling,yang2020closer,staib2017distributionally,Song2020Robust,xing2021generalization}. One line of works tries to extend the existing AT methods with regularizations to mitigate the tradeoff. For example, Song \textit{et al}. \cite{Song2020Robust} and Xing \textit{et al}. \cite{xing2021generalization} propose to improve the generalizability by introducing robust local features and $l_1$ penalty to AT, respectively. The other line aims at breaking the tradeoff by exploiting the distributional property of AEs and natural examples. For example, Yang \textit{et al}. \cite{yang2020closer} discover the $r$-separation phenomenon in widely-used image datasets, namely the distributions of different classes are at least distance 2$r$ apart from each other in pixel space. Based on this discovery, they theoretically prove that a classifier with sufficient local Lipschitz smoothness can achieve both high accuracy on natural examples and acceptable robustness on AEs with perturbations of size up to $r$. Differently, Stutz \textit{et al}. \cite{stutz2019disentangling} find that there exist AEs on the manifold of natural examples, and hence the adversarial robustness against the on-manifold AEs will equivalently improve the standard generalizability. Similarly, Staib \textit{et al}. \cite{staib2017distributionally} propose Distributionally Robust Optimization (DRO), which offers a more general perspective that the tradeoff can be eliminated by the alignment between the distributions of AEs and natural examples. Instead of generating point-wise perturbations with norm-constrained magnitude, DRO aims to seek worst-case AE distribution by restricting the distance between the AE distribution $\tilde{p}$ and the natural data distribution $p$ through the following min-max game:
\begin{equation}
\min_{C} \max_{P_W: W_{p}(P,P_W)\leq\epsilon, (\tilde{x},y)\sim P_W} \mathbb{E}_{(x,y) \sim)P} l(\tilde{x}, y; C),
\label{Eq:DRO}
\end{equation}
where $W_{p}$ is a distributional distance measure, e.g., Wasserstein distance. As $\epsilon$ approaching 0, $P_W$ is gradually aligning with $P$, and consequently, the tradeoff tends to disappear. Indeed, on-manifold AEs can be seen as a special case of DRO when $\epsilon = 0$.

Our work provides new arguments for the opinion that the tradeoff is eliminable. In particular, we make two contributions which distinguish our work from the existing works addressing the tradeoff: (1) We extend the research scope to UAEs by proposing a novel AT method PUAT; (2) By solid theoretical analysis and extensive empirical investigation, we show that the standard generalizability and comprehensive adversarial robustness are both achievable for PUAT.

\section{Conclusion}
In this paper, we propose a unique viewpoint that understands UAEs as imperceptibly perturbed unobserved examples, which rationalizes why UAEs exist and why they are able to deceive a well-trained classifier. This understanding allows us to view RAE as a special UAE, thus providing the feasibility of achieving comprehensive adversarial robustness against both UAE and RAE. Also, we find that if the UAE distribution and the natural data distribution can be aligned, the conflict between robust generalization and standard generalization in traditional AT methods can be eliminated, thus improving both the adversarial robustness and standard generalizability of a target classifier. Based on these ideas, we propose a novel AT method called Provable Unrestricted Adversarial Training (PUAT). PUAT utilizes partially labeled data to achieve efficient UAE generation by accurately capturing the real data distribution through a well-designed G-C-D GAN. At the same time, PUAT extends the traditional AT by introducing the supervised loss of the target classifier into the adversarial training and achieves alignment between the UAE distribution, the natural data distribution, and the distribution learned by the classifier, with the collaboration of the G-C-D GAN. The solid theoretical analysis and the extensive experiments demonstrate that PUAT can improve tradeoffs between adversarial robustness and standard generalizability through distributional alignment, and compared to the baseline methods, PUAT enables a target classifier to better defend against both UAE attacks and RAE attacks, as well as significantly increases its standard generalizability.


%

%

\ifCLASSOPTIONcompsoc
  \section*{Acknowledgments}
\else
  \section*{Acknowledgment}
\fi

This work is supported by National Natural Science Foundation of China under grant 61972270, and in part by NSF under grants III-1763325, III-1909323, III-2106758, and SaTC-1930941.

\ifCLASSOPTIONcaptionsoff
  \newpage
\fi



\bibliographystyle{IEEEtran}
\bibliography{UAR_PAMI_final.bib}
%

%

\begin{IEEEbiography}[{\includegraphics[width=1in,height=1.25in,clip,keepaspectratio]{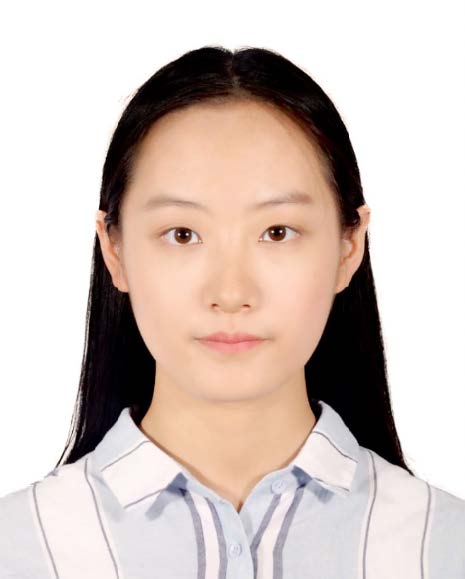}}]{Lilin Zhang}
obtained her bachelor's degree from the School of Information and Security Engineering, Zhongnan University of Economics and Law, in 2021. She is now pursuing  the master's degree in the School of Computer Science, Sichuan University, China. Her research interest focuses on adversarial machine learning and its applications. 
\end{IEEEbiography}

\begin{IEEEbiography}[{\includegraphics[width=1in,height=1.25in,clip,keepaspectratio]{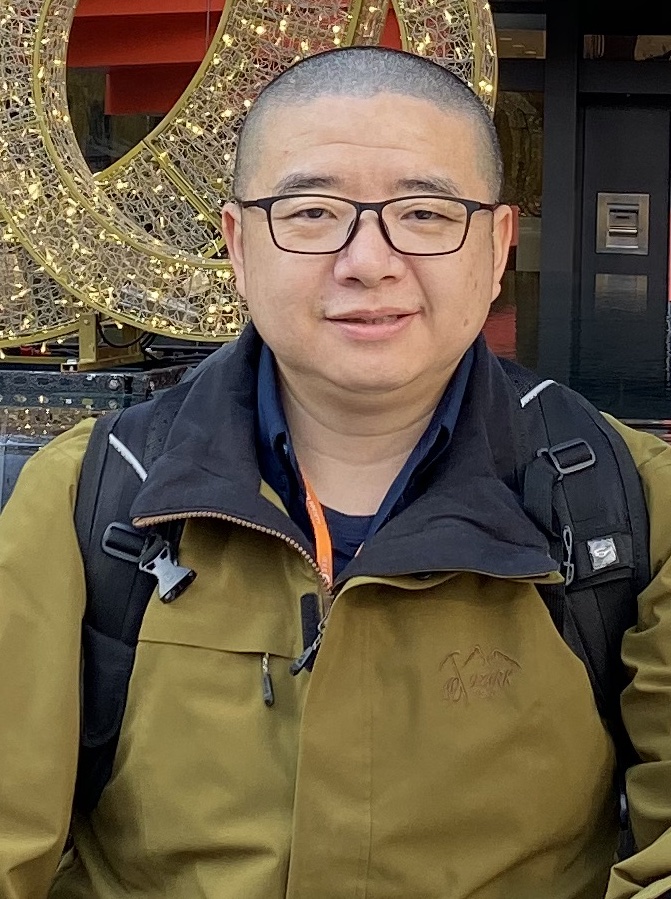}}]{Ning Yang}
is an associate professor at Sichuan University, China. He obtained his PhD degree in Computer Science from Sichuan University in 2010. His research interests include adversarial machine learning, graph learning, and recommender systems.
\end{IEEEbiography}

\begin{IEEEbiography}[{\includegraphics[width=1in,height=1.25in,clip,keepaspectratio]{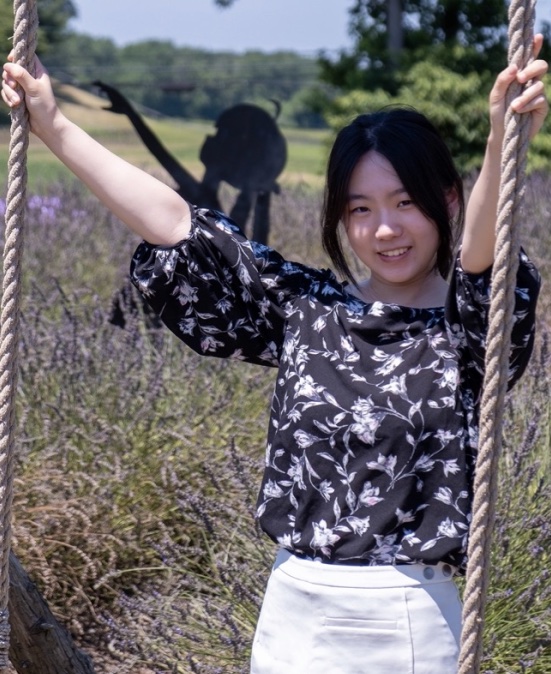}}]{Yanchao Sun}
is a Ph.D. candidate in University of Maryland, College Park, USA, where she is advised by Prof. Furong Huang. Her research mainly focuses on reinforcement learning (RL), including adversarial RL, multi-task RL, sample-efficient RL, etc. 
\end{IEEEbiography}


\begin{IEEEbiography}[{\includegraphics[width=1in,height=1.25in,clip,keepaspectratio]{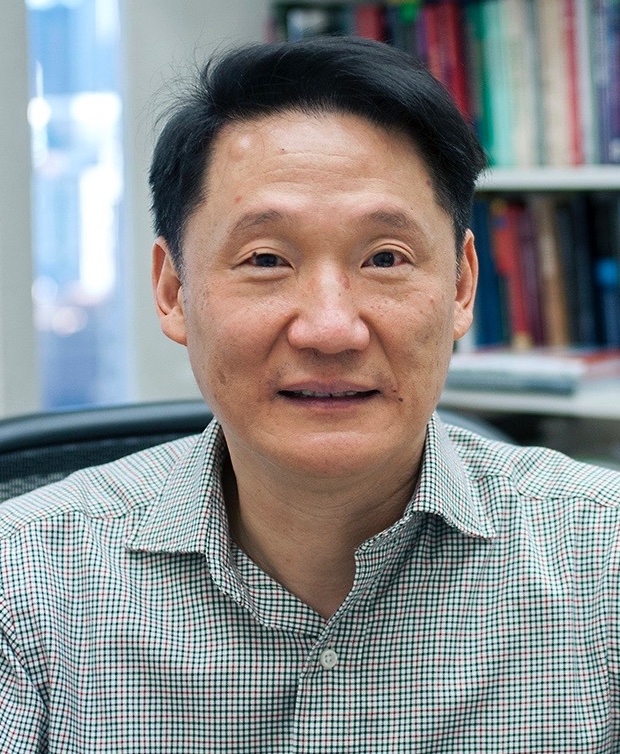}}]{Philip S. Yu}
received the PhD degree in electrical engineering from Stanford University. He is a distinguished professor in computer science at the University of Illinois at Chicago and is also the Wexler chair in information technology. His research interests include big data, data mining, and social computing. He is a fellow of the ACM and the IEEE.
\end{IEEEbiography}





\appendices

\section{Proof}
\begin{lemma}
The optimal solution of $\min_{C, G} \max_{D} \mathcal{L}^{G,C,D}_{gan}$ (Equation (11)) is achieved at $P(x,y)=(P_G (x, y) + P_C (x, y)) / 2$.
\label{Lemma:GCD}
\end{lemma}
\begin{proof}
Let $P_{GC} (x,y) = \frac{1}{2}(P_{G}(x,y)$ $ + P_{C}(x,y))$. According to Equations (7), (8), (10) and (11), it is easy to see
\begin{equation*}
\begin{aligned}
\mathcal{L}^{G,C,D}_{gan} = & \iint {P(x,y) D(x,y) - P_{GC}(x,y) D(x,y) \text{ dxdy}}.
\end{aligned}
\end{equation*}
For fixed $G$ and $C$, the optimal discriminator \cite{goodfellow2014generative} is
\begin{equation*}
D^* = \argmax_{D} \mathcal{L}^{G,C,D}_{gan} = \text{sign} \left( P - P_{GC} \right),
\end{equation*}
where $\text{sign}(x)$ is $1$ if $x > 0$, $0$ if $x = 0$, and $-1$ otherwise. Plugging $D^*$ in $\mathcal{L}^{G,C,D}_{gan}$, we can get
\begin{equation*}
\begin{aligned}
\mathcal{L}^{G,C,D^*}_{gan} = & \iint {(P - P_{GC}) D^* \text{ dxdy}}\\
= &\iint { \vert P - P_{GC} \vert \text{ dxdy}}\\
= & 2 \text{TV}(P, P_{GC}),
\end{aligned}
\end{equation*}
where $\text{TV}$ is Total Variation. Since $\text{TV}(P, P_{GC}) \ge 0$, $G^*$ and $C^*$ are achieved at $\text{TV}(P, P_{GC})$ $ = 0$, which leads to $P(x,y) = P_{GC} (x,y) = \frac{1}{2} \left(P_G (x, y) + P_C (x, y) \right)$. 
\end{proof}

\begin{lemma}
Let $\mathcal{Z}$ be the set of sampled noise  $z\sim P_z(z)$, $m$ $=\vert \mathcal{D}_l \vert$, $n=\vert \mathcal{D}_c \vert$, and $b = \Vert \log \frac{P}{P_{GC}} + 1 \Vert_\infty$ $ = \max_{(x, y)}$ $\log$ $ \frac{P}{P_{GC}} + 1$. For any $G$, $C$ and $0< \delta <1$, if $\vert\mathcal{Z}\vert \to \infty$, $\text{TV}(P, P_{GC}) \leq B$ holds with probability at least $(1 - \delta)^2$, where $B = \Big( \frac{1}{2}b
+\frac{1}{4} \max_{D} \hat{\mathcal{L}}^{G,C,D}_{gan} 
+ \sqrt{\frac{\log{\frac{1}{\delta}}}{8m}} + \sqrt{\frac{ \log{\frac{1}{\delta}}}{32n}} \Big)^{\frac{1}{2}}$.
\label{Lemma:GCD_finite}
\end{lemma}
\begin{proof}
Since the symmetry of total variation, we have
\begin{equation*}
\begin{aligned}
\text{TV}(P, P_{GC})^2
=& \frac{1}{2}\text{TV}(P, P_{GC})^2 + \frac{1}{2}\text{TV}(P_{GC}, P)^2.
\end{aligned}
\end{equation*}
According to Pinsker’s Inequality $\text{TV}(P_1, P_2)^2 \leq \frac{1}{2}$ $ \text{KL}(P_1 $ $\Vert P_2) $ \cite{cover1999elements}, we can obtain
\begin{equation*}
\begin{aligned}
\text{TV}(P, P_{GC})^2
\leq &\frac{1}{4} \text{KL}(P \Vert P_{GC}) + \frac{1}{4}\text{KL}(P_{GC}\Vert P) \\
=& \frac{1}{4} \iint (P - P_{GC})\log{\frac{P}{P_{GC}} } \text{ dxdy} \\
=& \underbrace{ \frac{1}{4} \iint (P - P_{GC})\left( \log{\frac{P}{P_{GC}}} - D \right)  \text{ dxdy}}_{\text{RHS}_1} \\ & + \underbrace{ \frac{1}{4} \iint (P - P_{GC}) D  \text{ dxdy} }_{\text{RHS}_2}, 
\end{aligned}
\end{equation*}
where $D$ represents the output of the discriminator $D(x, y)$. For the first part of the right-hand-side $\text{RHS}_1$, considering $\vert D \vert \le 1$, we have
\begin{equation*}
\begin{aligned}
\text{RHS}_1 \leq & \frac{1}{4} \iint \left\vert P - P_{GC} \right\vert \left\vert \log{ \frac{P}{P_{GC}} - D} \right\vert \text{ dxdy}\\ 
\leq & \frac{1}{4} \iint \left( \left\vert P \right\vert + \left\vert P_{GC} \right\vert \right) \left( \left\vert \log{ \frac{P}{P_{GC}} } \right\vert + \vert D \vert \right) \text{ dxdy}\\ 
\leq &  \frac{1}{4} \left( \left\Vert \log{ \frac{P}{P_{GC}} } \right\Vert_\infty + 1 \right) \iint \left( \left\vert P \right\vert + \left\vert P_{GC} \right\vert \right)  \text{ dxdy} \\ 
= & \frac{1}{2} \left\lVert \log{ \frac{P}{P_{GC}}} + 1 \right\rVert_\infty . 
\end{aligned}
\end{equation*}
At the same time, noting that $\text{RHS}_2 = \frac{1}{4} \mathcal{L}^{G,C,D}_{gan}$ according to Equations (7), (8), (10) and (11), hence
\begin{equation*}
\begin{aligned}
\text{TV}(P, P_{GC})^2
\leq \frac{1}{2} \left\lVert \log{ \frac{P}{P_{GC}}} + 1 \right\rVert_\infty + \frac{1}{4} \mathcal{L}^{G,C,D}_{gan}. 
\end{aligned}
\end{equation*}
Applying Hoeffding's Inequality \cite{hoeffding1994probability}, it is easy to show that with probability at least $1-\delta$,
\begin{equation*}
\begin{aligned}
\mathcal{L}_{D} - \hat{\mathcal{L}}_{D} & \leq 2 \sqrt{\frac{\log{\frac{1}{\delta}}}{2m}}, \\
\mathcal{L}_{C} - \hat{\mathcal{L}}_{C} & \leq 2 \sqrt{\frac{\log{\frac{1}{\delta}}}{2n}}, \\
\mathcal{L}_{G} - \hat{\mathcal{L}}_{G} & \leq 2 \sqrt{\frac{\log{\frac{1}{\delta}}}{2m\vert\mathcal{Z}\vert}}, 
\end{aligned}
\end{equation*}
where $\vert\mathcal{Z}\vert$ is the amount of sampled noises $z\sim P_z(z)$. Therefore if $\vert\mathcal{Z}\vert \to \infty$, $\hat{\mathcal{L}}_{G} \to \mathcal{L}_{G}$. Noting that $\mathcal{L}^{G,C,D}_{gan} = \mathcal{L}_{D} + \frac{1}{2} \mathcal{L}_{G} + \frac{1}{2} \mathcal{L}_{C}$ (Equation (11)), hence with probability at least $(1-\delta)^2$,
\begin{equation*}
\begin{aligned}
\mathcal{L}^{G,C,D}_{gan} \leq & \hat{\mathcal{L}}_{D} + 2 \sqrt{\frac{ \log{\frac{1}{\delta}}}{2m}} +  \frac{1}{2} \hat{\mathcal{L}}_{G} + \frac{1}{2} \left( \hat{\mathcal{L}}_{C} + 2 \sqrt{\frac{ \log{\frac{1}{\delta}}}{2n}} \right)\\
\leq & \hat{\mathcal{L}}^{G,C,D}_{gan} + 2 \sqrt{\frac{ \log{\frac{1}{\delta}}}{2m}} + \sqrt{\frac{ \log{\frac{1}{\delta}}}{2n}}.
\end{aligned}
\end{equation*}

Therefore, for any $0<\delta<1$, $G$, $C$ and $D$, the following inequality holds with probability at least $(1 - \delta)^2$: 
\begin{equation*}
\begin{aligned}
\text{TV}(P, P_{GC})^2 \leq \frac{1}{2} \left\lVert \log{ \frac{P}{P_{GC}}} + 1 \right\rVert_\infty +\frac{1}{4} \hat{\mathcal{L}}^{G,C,D}_{gan} \\
+ \sqrt{\frac{\log{\frac{1}{\delta}}}{8m}} + \sqrt{\frac{ \log{\frac{1}{\delta}}}{32n}}. 
\end{aligned}
\end{equation*}
Since the above Inequality holds for any $D$, we can choose the optimal $D^*=\argmax \hat{\mathcal{L}}^{G,C,D}_{gan}$ for any given $G$ and $C$. Then 
\begin{equation*}
\begin{aligned}
\text{TV}(P, P_{GC})^2
& \leq \frac{1}{2} \left\lVert \log{ \frac{P}{P_{GC}}} + 1 \right\rVert_\infty
+\frac{1}{4} \max_{D} \hat{\mathcal{L}}^{G,C,D}_{gan} \\
& + \sqrt{\frac{\log{\frac{1}{\delta}}}{8m}} + \sqrt{\frac{ \log{\frac{1}{\delta}}}{32n}} = B^2.
\end{aligned}
\end{equation*}
\end{proof}

\section{Hyper-parameter Setting and The Architecture}
\begin{table*}[t]
	\centering
	\caption{
	Hyper-parameter setting of the attacking methods for evaluation.
	}
	\label{Tb:attacks}
	\setlength{\tabcolsep}{0.05cm}
	\begin{tabular}{c|c|c}

	\toprule
	\multicolumn{2}{l|}{Attacks}&\makecell[l]{Hyper-parameter setting}\\
	\midrule  
	\multirow{5}{*}{RAE}&\makecell[l]{PGD-8/255}&\makecell[l]{$\epsilon=8/255$; step size = 1/255; number of steps = 20.}\\
	&\makecell[l]{PGD-4/255}&\makecell[l]{$\epsilon=4/255$; step size = 1/255; number of steps = 20.}\\
	&\makecell[l]{PGD-2/255}&\makecell[l]{$\epsilon=2/255$; step size = 1/255; number of steps = 20.}\\
	\cmidrule{2-3}    
	&\makecell[l]{AA}&\makecell[l]{$\epsilon=8/255$; version=standard.}\\
	\midrule
	\multirow{5}{*}{\makecell[c]{UAE}}&\makecell[l]{GPGD-0.1}&\makecell[l]{$\epsilon=0.1$; step size = 0.1; number of steps = 20.}\\
	&\makecell[l]{GPGD-0.01}&\makecell[l]{$\epsilon=0.01$; step size = 0.1; number of steps = 20.}\\
	\cmidrule{2-3}  
	&\makecell[l]{USong}&\makecell[l]{$\epsilon=0.01$; step size = 0.1; number of steps = 200; \\$\lambda_1=100, \lambda_2=100$, version=untargeted.}\\
	\bottomrule

	\end{tabular}
\end{table*}

\begin{table*}[t]
	\centering
	\caption{Hyper-parameter setting of the baseline methods.}
	\label{Tb:hyper}
	\setlength{\tabcolsep}{0.05cm}
	\begin{tabular}{l|l|l|l|l|l}
	\toprule
	{Methods}&{Tiny ImageNet}&{ImageNet32}&{SVHN}&{CIFAR10}&{CIFAR100}\\
	\midrule
	{TRADES}&{$\beta=5.0$}&{$\beta=5.0$}&{$\beta=5.0$}&{$\beta=5.0$}&{$\beta=1.0$}\\
	\midrule
	{DMAT}&{$\beta=5.0$}&{-}&{$\beta=5.0$}&{$\beta=5.0$}&{$\beta=1.0$}\\
	\midrule
	{RST}&{$\beta=5.0$}&{$\beta=5.0$}&{$\beta=5.0$}&{$\beta=5.0$}&{$\beta=1.0$}\\
	\midrule
	{PUAT}&\makecell[l]{$\lambda=10.0$\\$\beta=6.0$\\$\gamma=0.03$\\$\alpha=50$}&\makecell[l]{$\lambda=10.0$\\$\beta=6.0$\\$\gamma=0.03$\\$\alpha=50$}&\makecell[l]{$\lambda=10.0$\\$\beta=6.0$\\$\gamma=0.03$\\$\alpha=50$}&\makecell[l]{$\lambda=10.0$\\$\beta=6.0$\\$\gamma=0.03$\\$\alpha=50$}&\makecell[l]{$\lambda=10.0$\\$\beta=10.0$\\$\gamma=0.03$\\$\alpha=500$}\\
	\bottomrule
	\end{tabular}
\end{table*}
\begin{table*}[!t]
	\renewcommand\arraystretch{}
	\centering
	\caption{Model architecture of $D$, $G$ and $A$ of PUAT.}
	\label{Tb:arch}
	\setlength{\tabcolsep}{0.2cm}
	\begin{tabular}{c|c|c}
	\toprule
	{Discriminator D}&{Generator G}&{Attacker A}\\
	\midrule  

	\makecell[c]{Input: $3\times 32\times 32$ image $x$, label $y$}
	&\makecell[c]{Input: noise $z$, label $y$}
	&\makecell[c]{Input: noise $z$, label $y$}
	\\ \midrule

	\makecell[c]{2-layer $3\times 3$ residual block, \\in-ch 3, out-ch 256, padding 1,\\ ReLU, spectral norm, downsample}
	&\makecell[c]{MLP 4096 units\\ reshape $256\times 4\times 4$}
	&\makecell[c]{MLP $256$ units\\ reshape $16\times 4\times 4$}
	\\ \cmidrule{0-2}

	\makecell[c]{2-layer $3\times 3$ residual block, \\in-ch 256, out-ch 256, padding 1,\\ ReLU, spectral norm, downsample}
	&\makecell[c]{2-layer $3\times 3$ residual block, \\in-ch 256, out-ch 256, padding 1,\\ ReLU, batch norm, upsample}
	&\makecell[c]{2-layer $3\times 3$ residual block, \\in-ch 16, out-ch 32, padding 1,\\ ReLU, batch norm}
	\\ \cmidrule{0-2}

	\makecell[c]{2-layer $3\times 3$ residual block, \\in-ch 256, out-ch 256, padding 1,\\ ReLU, spectral norm, downsample}
	&\makecell[c]{2-layer $3\times 3$ residual block, \\in-ch 256, out-ch 256, padding 1,\\ ReLU, batch norm, upsample}
	&\makecell[c]{2-layer $3\times 3$ residual block, \\in-ch 32, out-ch 16, padding 1,\\ ReLU, batch norm}
	\\ \cmidrule{0-2}

	\makecell[c]{2-layer $3\times 3$ residual block, \\in-ch 256, out-ch 256, padding 1,\\ ReLU, spectral norm, downsample}
	&\makecell[c]{2-layer $3\times 3$ residual block, \\in-ch 256, out-ch 256, padding 1,\\ ReLU, batch norm, upsample}
	&{-}
	\\ \midrule

	\makecell[c]{MLP 1 units, spectral norm}
	&\makecell[c]{$1\times 1$ conv, in-ch 256, out-ch 3, \\tanh, batch norm}
	&\makecell[c]{MLP $256$ units, ReLU}

	\\ \bottomrule

	\end{tabular}
\end{table*}


\end{document}